\documentclass{article}

\PassOptionsToPackage{numbers}{natbib}

 \usepackage[preprint]{neurips_2025}

\usepackage{algorithm}
\usepackage{algpseudocode}
 \usepackage{graphicx}      
  \usepackage{subcaption}    
\usepackage{xcolor}
\usepackage{multirow}
\usepackage[most]{tcolorbox}
\usepackage{xcolor}
\usepackage{microtype}

\usepackage{amsmath}
 \usepackage{graphicx}
 \usepackage[utf8]{inputenc} 
\usepackage[T1]{fontenc}    
\usepackage{hyperref}       
\usepackage{url}            
\usepackage{booktabs}       
\usepackage{amsfonts}       
\usepackage{nicefrac}       
\usepackage{microtype}      
\usepackage{xcolor}         
\usepackage{multirow}
\usepackage{amssymb}
\usepackage{float}
\usepackage{wrapfig}

\usepackage{amsthm}
\usepackage{tcolorbox}
\tcbuselibrary{skins, breakable}
\usepackage{listings}
\usepackage{xcolor}

\newtheorem{prop}{Proposition}
\newtheorem*{prop*}{Proposition}
\newtheorem{lem}{Lemma}

\newcommand\blfootnote[1]{%
  \begingroup
  \renewcommand\thefootnote{}\footnote{#1}%
  \addtocounter{footnote}{-1}%
  \endgroup
}
\usepackage{xcolor}

\title{AgentKVShift: Efficient KV Cache Reuse for Agentic Memory Systems}

\author{%
  Nilesh Prasad Pandey$^\dagger$ \quad
  Jason Kong \quad
  Lanxiang Hu \quad
  Quanling Zhao \\
  \textbf{Yujie Zhao} \quad
  \textbf{Onat Gungor} \quad
  \textbf{Hao Zhang} \quad
  \textbf{Tajana Rosing} \\
  University of California, San Diego \\
}

\begin{document}
\maketitle
\blfootnote{$^{\dagger}$Corresponding author: \texttt{nppandey@ucsd.edu}}




\begin{abstract}
Memory-augmented LLM agents have gained substantial attention recently for their ability to maintain context across hundreds of interactions through agentic memory systems that actively curate retrieved content with LLM-generated metadata such as summaries, keywords, and tags. However, from an inference cost standpoint, every retrieval triggers a full
re-encoding of these structured memory units into Key-Value (KV) states, which dominates prefill latency. Existing training-free KV reuse methods mitigate this by selectively recomputing a small fraction of tokens, but were designed for RAG-style raw passages and degrade substantially on structured agentic memories. In this work, we present AgentKVShift, a training-free, probe-guided KV residual correction method that operates per retrieved memory unit. One of the crucial insights we demonstrate is that the per-memory KV reuse residual decomposes into a shared memory-level offset plus small token-wise fluctuations. Estimating this offset from a small probe set allows us to correct every
reused token in the memory unit by a single weighted correction. Unlike prior reuse methods which decide which tokens to recompute and leave the rest of the cache stale, AgentKVShift also corrects the tokens it does not recompute, turning the refresh budget into useful signal across the entire chunk. Through extensive experiments across four open source LLMs spanning 3B to 32B parameters and two long-horizon agentic memory benchmarks covering long-term dialogue and agentic applications, we show that AgentKVShift achieves near full recompute performance while refreshing only 10--30\% of the cache,
outperforming existing baselines at the same recompute ratio. AgentKVShift requires up to 5$\times$ lower recompute to reach this near-full performance, which prior reuse methods only attain at 45--55\% refresh. By operating in this lower recompute regime, AgentKVShift delivers prefill speedups of 2--3.5$\times$ over no-KV-reuse on a single A100 GPU. Lastly, AgentKVShift orthogonally composes with KV cache quantization, retaining over 2$\times$
the F1 of prior reuse methods under aggressive 2- and 4-bit settings, making it a simple, yet effective choice for serving long-horizon agentic memory workloads.
\end{abstract}

\section{Introduction}

Memory-augmented LLM agents, such as personal assistants, coding agents, and 
enterprise knowledge systems, have gained substantial attention recently. 
These agents maintain context across hundreds of interactions through memory 
systems that selectively retain, organize, and retrieve information over 
extended periods~\cite{episodic_memory}. While initial progress relied on 
Retrieval-Augmented Generation (RAG)~\cite{rag}, which retrieves raw 
conversation passages, the field has rapidly shifted toward 
\textit{agentic memory} systems~\cite{memoryr1, mem1, xu2025mem, gmemory} that 
actively curate knowledge through learned or heuristic-driven policies. 
This shift introduces a fundamental change in the nature of the retrieved 
context: rather than raw text, the retrieved representations are enriched 
with LLM-generated metadata such as summaries, keywords, tags, and relational 
annotations. Crucially, 
this metadata is not overhead, it is what the model primarily attends to 
during answer generation.

A key computational challenge for such memory systems is the prefill cost. 
Every time an agent retrieves stored memory, it must encode the text into 
Key-Value (KV) states, a computation that scales linearly with context length 
and dominates inference latency for long-context workloads~\cite{agrawal2024taming, kwon2023efficient}. Considerable 
effort has gone into reducing this cost through KV cache reuse, ranging from 
prefix caching in modern serving frameworks, vLLM~\cite{kwon2025vllm} and SGlang~\cite{zheng2024sglang}, to more 
targeted approaches that selectively recompute a small fraction of KV states 
under dynamic retrieval~\cite{cacheblend, wang2026prophetkv}. These methods have 
proven effective on RAG workloads, where retrieved passages are raw text and 
the dominant source of stale KV states is well-localized to a small subset 
of tokens. 

However, agentic memory presents a structurally different setting, and we 
find that existing reuse methods fail to transfer to it. The accuracy 
degradation is substantial and consistent across model scales from 3B to 
32B parameters, across multiple model families (Qwen2.5 \cite{qwen2.5}, Qwen3 \cite{qwen3technicalreport}, Mistral \cite{mistralai2024mistral7binstructv03}), 
and across different agentic memory systems,  
suggesting that the issue is not an artifact of a particular model or 
a memory type but a generalized failure mode of token-selection-based reuse on 
metadata-rich retrieval. Furthermore, there remains little understanding of 
how the token composition of metadata-rich agentic memory affects KV 
divergence patterns and, consequently, the token-selection strategies that 
underlie current reuse methods.

In view of the above, we seek to address the following key questions: (1) Do existing KV cache reuse methods generalize to agentic memory systems? (2) If they degrade, does the KV reuse error have a shared structure across tokens that current reuse methods cannot fix? (3) Can we fix this failure with a simple, training-free correction?

To this end, we propose \textit{AgentKVShift}, a training-free, probe-guided KV 
residual correction that operates per retrieved memory unit. The crucial insight 
we demonstrate is that the per memory unit KV reuse residual decomposes into a 
shared chunk-level offset plus small token-wise fluctuations, estimating 
the offset from a small probe set allows us to correct every other token 
in the chunk by a single weighted vector addition, recovering near-fresh 
attention outputs without re-encoding the memory. Unlike prior reuse 
methods~\cite{cacheblend,wang2026prophetkv} which decide \emph{which} tokens to recompute and leave the rest 
of the cache stale, AgentKVShift also corrects the keys and values of 
the tokens it does not recompute, turning the refresh budget into useful 
signal across the entire chunk.

Overall, we make the following key contributions:
\begin{enumerate}
    \item We show that existing training-free KV cache reuse methods 
          degrade substantially on agentic memory retrieval, with the 
          accuracy gap persisting across three model families, scales 
          from 3B to 32B, and both note- and graph-based memory systems.
    \item Through a per-memory spectral analysis, we identify the 
          structural cause: the KV reuse residual is dominated by a 
          shared memory-level offset that token-selection methods leave 
          intact in unrefreshed tokens. Motivated by this, we propose 
          \textit{AgentKVShift}, a training-free correction that estimates 
          the offset from a small probe set and applies a weighted 
          mean-shift to all reused tokens. Under standard sub-Gaussian 
          assumptions, AgentKVShift admits a tighter attention-error bound 
          than uncorrected reuse when the memory-level bias dominates 
          the probe estimation error.
    \item On diverse long-horizon agentic memory benchmarks~\cite{maharana2024evaluating,zhao2026ama}, AgentKVShift 
          narrows the gap to full recompute to within $1.5-6\%$ 
          relative F1 while refreshing only $10\%$ of the cache, 
          attaining prefill speedups of $2-3.5\times$ on a single A100 
          GPU. Also, under state-of-the-art 2- and 4-bit KV quantization~\cite{liu2024kivi,su2025accurate}, 
          AgentKVShift retains over $2\times$ the F1 of existing reuse 
          methods at aggressive bit-widths, since the probe estimate 
          absorbs both context drift and average quantization error.
\end{enumerate}

\section{Related Work}
\label{sec:related_work}

\textbf{RAG and Its Limitations for Long-Horizon Agents.}
Retrieval-Augmented Generation (RAG)~\cite{rag} retrieves relevant passages 
from a document store and prepends them to the LLM prompt at inference time, 
relying on the retriever (e.g., BM25~\cite{bm25}, dense retrieval~\cite{karpukhin2020dense, izacard2021unsupervised}) to 
surface relevant content. While effective for single-turn 
knowledge-intensive tasks, RAG has well-documented limitations in 
long-horizon agentic settings~\cite{episodic_memory}. As conversations grow 
to hundreds of turns, the raw passage store becomes noisy, retrieval 
recall degrades, and the retrieved context lacks the structure needed for 
multi-hop reasoning across sessions. These limitations have motivated the 
shift toward agentic memory systems that actively curate and organize 
knowledge.

\textbf{Agentic Memory Systems.}
A growing body of work has developed agentic memory systems that retain 
the retrieve-and-condition paradigm of RAG but replace raw passage stores 
with structured, LLM-curated representations. Although the systems vary 
in their specific organizational primitives, two dominant categories have 
emerged. \textit{Note-based} systems~\cite{xu2025mem} generate Zettelkasten-style 
notes enriched with summaries, keywords, and tags, building on a longer 
line of work in dialogue compression and fact-extraction memory that 
distill interactions into compact, self-contained 
units~\cite{lee2024human, zhong2024memorybank,packer2023memgpt}. On the other hand, \textit{Graph-based} 
systems~\cite{huang2025licomemory} organize memories into hierarchical or relational 
graphs with temporal indexing, drawing on prior work in graph-augmented 
retrieval~\cite{rasmussen2025zep,chhikara2025mem0}. At inference time, both categories operate by retrieval. The agent's query 
is matched against the curated memory store, and the retrieved units are 
passed to the LLM as context. The difference from RAG lies in what gets 
retrieved, not raw passages, but LLM-curated representations such as 
notes, summaries, tags, and graph nodes. Despite their diversity, these systems 
share a common property, they excel at determining \textit{what} to 
remember, but treat the underlying KV computation as a black box, 
requiring full re-encoding on every retrieval and leaving the prefill cost 
unaddressed.

\textbf{KV Cache Reuse for Dynamic Retrieval.}
Reducing prefill computation through KV cache reuse has been studied at 
multiple levels. At the serving level, prefix caching in vLLM~\cite{kwon2025vllm} 
and RadixAttention in SGLang~\cite{zheng2024sglang} cache KV states for shared 
prompt prefixes using block-level hashing or radix tree structures. These 
methods are effective when prompts share long, stable prefixes (e.g., 
system prompts, few-shot examples), but offer limited benefit when both 
the content and the ordering of the retrieved context vary per request, as is typical in retrieval- and memory-augmented agents, where each query 
retrieves a different set of memories arranged in a different order, 
breaking the prefix-caching assumptions. For dynamic retrieval settings, more targeted 
training-free methods have been proposed. CacheBlend~\cite{cacheblend} 
identifies high-deviation KV tokens at a check layer and selectively 
recomputes only those, achieving near-baseline accuracy at recompute 
ratios of $r = 0.10$--$0.20$ on RAG benchmarks. More recently, 
ProphetKV~\cite{wang2026prophetkv} addresses a limitation of divergence-based 
selection within RAG by replacing the divergence criterion with a 
query-driven one. Despite this progress, all of 
the selective-recomputation methods above were designed and evaluated on 
RAG or standard generation workloads. Our work shows that the structural properties of the KV reuse error in agentic memory call for a different 
correction strategy, one that adjusts the values of unrefreshed tokens 
rather than only choosing which tokens to refresh~\cite{cacheblend,wang2026prophetkv} and adjusting its positional embedding~\cite{gim2024prompt}.

\section{Proposed Approach}
\label{sec:method}

\setlength{\abovedisplayskip}{3pt}
\setlength{\belowdisplayskip}{3pt}
\setlength{\abovedisplayshortskip}{2pt}
\setlength{\belowdisplayshortskip}{2pt}

\subsection{Motivation for KV Reuse}
\label{sec:motivation}
\begin{figure}[t]
    \centering
    \includegraphics[height=4cm]{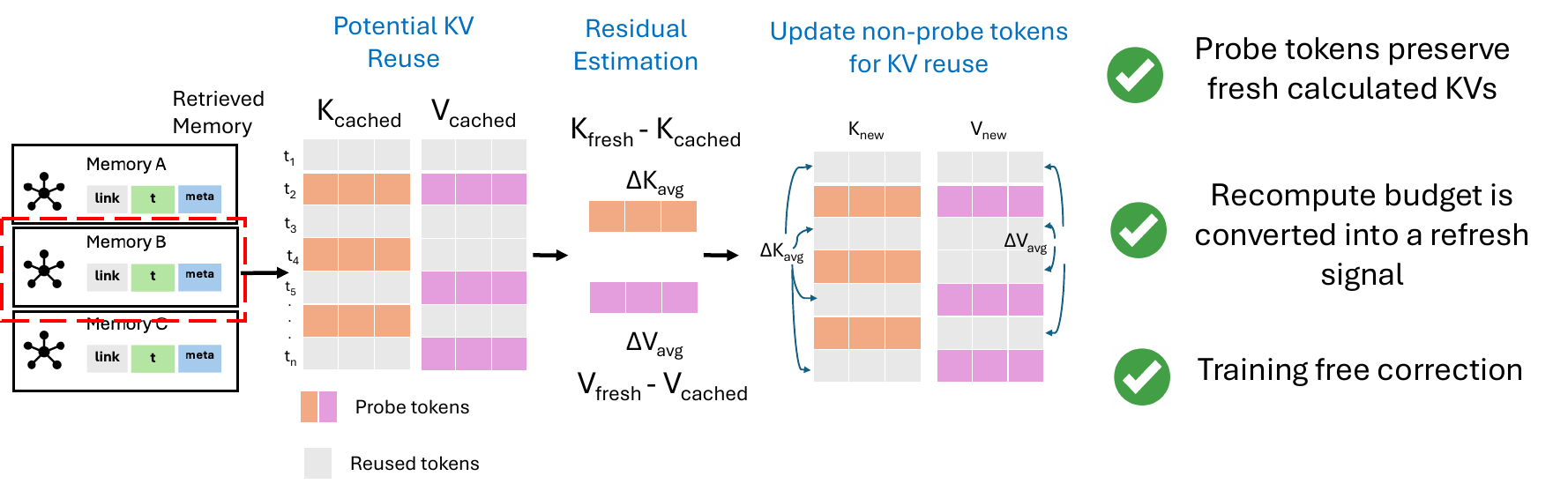}
    \caption{AgentKVShift: KV reuse via residual estimation. The mean residuals $\Delta K_\text{avg}$ and $\Delta V_\text{avg}$ are estimated using probe tokens are applied to non-probe tokens to produce $K_\text{new}$ and $V_\text{new}$, approximating full recomputation.}
\end{figure}

KV caching is an established technique for accelerating LLM prefill and inference. Precomputed key/value states for repeated context can be loaded directly instead of being recomputed at every step. However, in agentic memory retrieval, this assumption breaks. Specifically, a memory unit is encoded once when the agent first observes it, but is later retrieved under a different query, instruction, or dialogue history. The tokens are the same, but their contextual representations are not.

To formalize this mismatch,
Let $\mathcal{C}$ denote the set of retrieved memory chunks, and let $C\in\mathcal{C}$ denote one retrieved memory chunk with $n$
tokens with head dimension $d$, processed through a transformer with $L$ layers. At layer $\ell$, direct reuse substitutes the cached states $K^\ell_{\mathrm{reuse}} = K^\ell_{\mathrm{cached}}$ and $V^\ell_{\mathrm{reuse}} = V^\ell_{\mathrm{cached}}$ in place of the fresh states $K^\ell_{\mathrm{fresh}}, V^\ell_{\mathrm{fresh}} \in \mathbb{R}^{n\times d}$ under the current context, introducing residuals $R_K^\ell = K^\ell_{\mathrm{fresh}} - K^\ell_{\mathrm{reuse}}$ and $R_V^\ell = V^\ell_{\mathrm{fresh}} - V^\ell_{\mathrm{reuse}}$ that distort attention and degrade generation quality even when the retrieved memory is semantically relevant. Our goal is therefore to approximate $K^\ell_{\mathrm{fresh}}, V^\ell_{\mathrm{fresh}}$ while recomputing only a small fraction of tokens per memory. To this end, we first study the structure of $R_K^\ell, R_V^\ell$ in section~\ref{sec:spectral} and then propose our correction scheme in section~\ref{sec:correction}.

\newtcolorbox{obsbox}{
    colback=gray!5,
    colframe=black,
    boxrule=0.6pt,
    arc=1pt,
    left=6pt, right=6pt, top=3pt, bottom=3pt,
    breakable
}

\subsection{Residual Structure via Spectral Analysis}
\label{sec:spectral}
\begin{wrapfigure}[10]{r}{0.5\linewidth}
    \vspace{-0.5in}

    \centering
    \includegraphics[width=\linewidth]{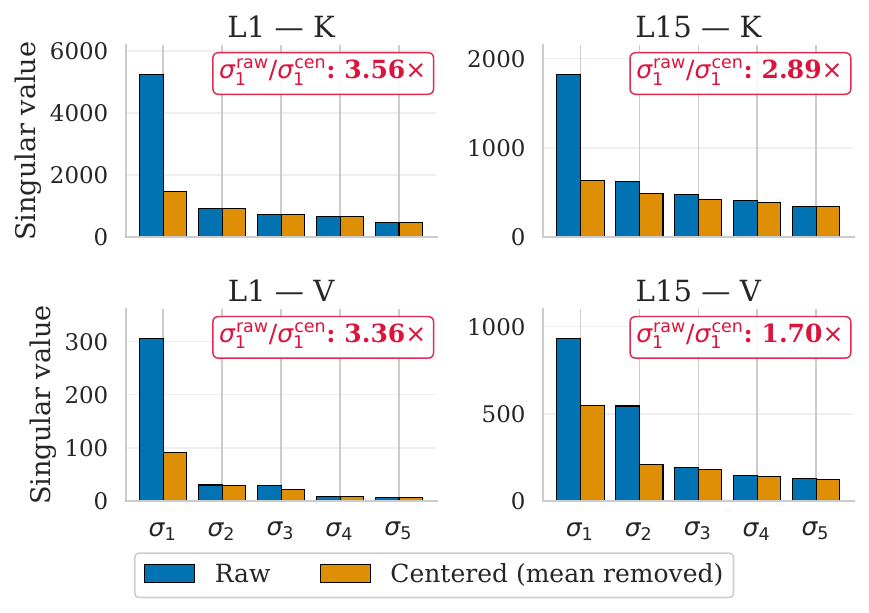}
    \caption{\small SVD spectra of $R^\ell$ for K and V.}
    \label{fig:svd}
\end{wrapfigure}
To understand the structure of the reuse error, we analyze the memory-wise KV residuals $R_K^\ell$ and $R_V^\ell$ at different transformer layers. We compute their singular value spectra and contrast them against the spectra obtained after subtracting the chunk mean from each row (Fig.~\ref{fig:svd}).

\noindent\textbf{Sharp spectral concentration.} For both $R_K^\ell$ and $R_V^\ell$, the leading singular value dominates the remaining modes across layers, indicating that the reuse error carries structured signal rather than arbitrary token-wise noise.

\noindent\textbf{Mean-centering removes the dominant mode.} Subtracting the chunk mean from each token's KV substantially reduces the leading singular value, especially for K. A large fraction of the reuse error is therefore captured by a single shared chunk-level offset, while the remainder corresponds to smaller token-wise fluctuations.

Together, these observations motivate our method: estimate that offset from a small probe set and shift the reused cache accordingly, enabling KV cache reuse.

\subsection{Probe-Guided Residual Correction}
\label{sec:correction}

Building on the above analysis, we now describe our correction scheme. For each retrieved memory chunk, we recompute fresh KV states for a small probe set of tokens, use them to estimate the chunk-level offset at every layer, and apply this offset to the remaining tokens via a single weighted vector addition. Algorithm~\ref{alg:chunk_residual_correction} summarizes the full procedure. The method has three steps, which are described below.

\textbf{Probe selection.} For each chunk $C$, we compute K- and V-specific divergence scores at the first transformer layer $\ell_c=1$, following the cross-layer high-deviation observation in CacheBlend~\citep{cacheblend}. Concretely,
{\small
\begin{equation*}
d_{K,j}^{\ell_c} = \|K_{j,\mathrm{fresh}}^{\ell_c} - K_{j,\mathrm{reuse}}^{\ell_c}\|_2, \qquad d_{V,j}^{\ell_c} = \|V_{j,\mathrm{fresh}}^{\ell_c} - V_{j,\mathrm{reuse}}^{\ell_c}\|_2.
\end{equation*}
}%
The top-$b$ probe sets $S_C^K, S_C^V$ and per-token weights $w_{K,j}=\min(d_{K,j}^{\ell_c},1)$, $w_{V,j}=\min(d_{V,j}^{\ell_c},1)$ are determined once at layer~1 and reused at every subsequent layer. This avoids per-layer probe re-selection and keeps the overhead low. Given recompute ratio $r$, we set the per-chunk probe budget to $b=\lceil r n \rceil$.

\textbf{Mean estimation.} At each layer $\ell$, the memory-level offset is estimated from the probes as
{\small
\begin{equation*}
\hat{\mu}_{K,C}^\ell = \tfrac{1}{|S_C^K|}\!\!\sum_{j\in S_C^K}\!(K_{j,\mathrm{fresh}}^\ell - K_{j,\mathrm{reuse}}^\ell), \qquad \hat{\mu}_{V,C}^\ell \text{ analogously.}
\end{equation*}
}%
Note that although the probe indices are fixed across layers, the estimated offsets $\hat{\mu}_{K,C}^\ell$ and $\hat{\mu}_{V,C}^\ell$ are layer-specific, since the residual structure itself varies with depth.
\begin{algorithm}[t]
\caption{AgentKVShift: Probe-Guided Residual Correction for KV Reuse}
\label{alg:chunk_residual_correction}
\begin{algorithmic}[1]
\Require Chunks $\mathcal{C}$, reused KV $\{K^\ell_{j,\mathrm{reuse}}, V^\ell_{j,\mathrm{reuse}}\}$, probe budget $b$, checkpoint layer $\ell_c$
\ForAll{$C \in \mathcal{C}$}
    \State Compute $d_{K,j}^{\ell_c}, d_{V,j}^{\ell_c}$ for $j\in C$, select top-$b$ probe sets $S_C^K, S_C^V$, set weights $w_{K,j}, w_{V,j}$
    \For{$\ell = 1, \dots, L$}
        \State Recompute fresh KV for $S_C^K \cup S_C^V$ at layer $\ell$
        \State $\hat\mu_{K,C}^\ell \gets \tfrac{1}{b}\sum_{j\in S_C^K}(K^\ell_{j,\mathrm{fresh}} - K^\ell_{j,\mathrm{reuse}})$, and $\hat\mu_{V,C}^\ell$ analogously
        \ForAll{$j \in C$}
            \State $K^\ell_{j,\mathrm{corr}} \gets K^\ell_{j,\mathrm{fresh}}$ if $j\in S_C^K$, else $K^\ell_{j,\mathrm{reuse}} + w_{K,j}\hat\mu_{K,C}^\ell$
            \State $V^\ell_{j,\mathrm{corr}} \gets V^\ell_{j,\mathrm{fresh}}$ if $j\in S_C^V$, else $V^\ell_{j,\mathrm{reuse}} + w_{V,j}\hat\mu_{V,C}^\ell$
        \EndFor
    \EndFor
\EndFor
\State \Return $\{K^\ell_{j,\mathrm{corr}}, V^\ell_{j,\mathrm{corr}}\}$
\end{algorithmic}
\end{algorithm}

\textbf{AgentKVShift correction.} Given the estimated offsets, each non-probe token is corrected by a single weighted vector addition $K_{j,\mathrm{corr}}^\ell = K_{j,\mathrm{reuse}}^\ell + w_{K,j}\hat{\mu}_{K,C}^\ell$, and analogously for V. Probe tokens use their freshly recomputed states directly. The weights $w_{K,j}, w_{V,j}$ ensure that tokens whose cached states already match the fresh states receive a smaller correction, while tokens with high check-layer divergence receive a stronger one.

\subsection{Theoretical Analysis}
\label{sec:theory}

\setlength{\abovedisplayskip}{3pt}
\setlength{\belowdisplayskip}{3pt}
\setlength{\abovedisplayshortskip}{2pt}
\setlength{\belowdisplayshortskip}{2pt}

Our correction method is based on a simple hypothesis about KV reuse error: within a retrieved chunk, the residual is often dominated by a shared chunk-level offset. For clarity, suppress the layer indices. For key and value residuals $r_i^K := k_i^\star-k_i^{\mathrm{reuse}}$ and $r_i^V := v_i^\star-v_i^{\mathrm{reuse}}$ ($*$ denotes fresh recomputed vector). Thus \(r_i^K\) and \(r_i^V\) correspond to rows of $R^\ell_{K}$ and $R^\ell_{V}$, respectively. We write $r_i^K=\mu_K+\xi_i^K$ and $r_i^V=\mu_V+\xi_i^V$, where $\mu_K,\mu_V$ are chunk-level offsets and $\xi_i^K,\xi_i^V$ are token-wise fluctuations with sub-Gaussian scales $\sigma_K,\sigma_V$. This motivates estimating the shared offset and applying an AgentKVShift correction.

Consider a fixed query $q \in \mathbb{R}^d$, and let $y^\star, y^{\mathrm{reuse}}, y^{\mathrm{corr}}$ denote the attention outputs under fresh, reused, and corrected KV with probe set of size $b$ uniformly sampled with replacement. Let $\|v_i^\star\|_2 \le V_{\max}$ for all $i$, and define $C_K = V_{\max}\|q\|_2/(2\sqrt{d})$ and $C_V = 1$. With this setup, and assuming correcting with the unweighted residual mean for K and V, the following result bounds the attention error before and after correction.

\begin{prop*}{\textbf{1 \& 2} (Reuse vs.\ corrected error)}
\label{prop:bounds}
With probability at least $1-\delta$ each,
{\small
\begin{equation*}
\|y^{\mathrm{reuse}} - y^\star\|_2 \le \underbrace{C_K\|\mu_K\|_2 + C_V\|\mu_V\|_2}_{\text{common bias }B} + \underbrace{(C_K\sigma_K + C_V\sigma_V)\sqrt{2\log\tfrac{4n}{\delta}}}_{N_n} \quad \text{(Proof in Appendix~\ref{prop1})}
\end{equation*}
\begin{equation*}
\|y^{\mathrm{corr}} - y^\star\|_2 \le \underbrace{(C_K\sigma_K + C_V\sigma_V)\sqrt{2\log\tfrac{8n}{\delta}}}_{N_n'} + \underbrace{(C_K\sigma_K + C_V\sigma_V)\sqrt{\tfrac{2\log(8/\delta)}{b}}}_{\text{probe error }E_b} \quad \text{(Proof in Appendix~\ref{prop2})}
\end{equation*}
}%
\end{prop*}
Thus, before correction the reuse error behaves like $\text{common bias} + \text{token-wise fluctuation}$ and after correction, the error behaves like $\text{probe estimation error} + \text{token-wise fluctuation}$. To compare the two bounds, let $U_{\mathrm{reuse}} = B + N_n$ and $U_{\mathrm{corr}} = N_n' + E_b$ and $U_{\mathrm{reuse}}-U_{\mathrm{corr}} = B - E_b - (N_n'-N_n)$.

\begin{prop*}{\textbf{3} (Sufficient condition for tighter upper bound)}
\label{prop:tighter}
On the event where the bounds in Propositions 1 and 2 both hold:
\begin{equation*}
\small
\begin{split}
\text{If} \quad B > E_b + (N_n'-N_n) \quad \text{then} \quad U_{\mathrm{corr}}<U_{\mathrm{reuse}} \quad \text{(Proof in Appendix~\ref{prop3})}
\end{split}
\end{equation*}
\end{prop*}
This means correction yields a strictly tighter high-probability upper bound whenever the common-bias term $B$ exceeds the probe-estimation error term $E_b$, up to the lower-order fluctuation-mismatch term $N_n'-N_n$. Overall, the theory says that naive reuse incurs both common bias and token-wise fluctuation, while probe-guided mean correction removes the common bias and replaces it with a smaller probe-estimation term that decreases with $b$. For completeness, we also empirically justify our structural modeling of residuals (dominated by a common offset) behind the theory and method, and verify whether $\xi_i^K$ and $\xi_i^V$ behave consistently with the sub-Gaussian-style concentration assumption used here. Those results are deferred to Appendix~\ref{mean_explained} and~\ref{fluctuation_validation}.

\section{Experiments}
\subsection{Experimental Setup}
\begin{table*}[t]
  \centering
  \caption{F1, ROUGE and BLEU-1 scores on LoCoMo using recompute ratio, $r{=}0.1$. Best KV-reuse result is \textbf{bold},}
  \label{tab:results}
  \resizebox{\textwidth}{!}{%
  \begin{tabular}{ll ccc ccc}
  \toprule
  & & \multicolumn{3}{c}{\textbf{AMem~\cite{xu2025mem}}} & \multicolumn{3}{c}{\textbf{LicoMemory~\cite{huang2025licomemory}}} \\
  \cmidrule(lr){3-5} \cmidrule(lr){6-8}
  \textbf{Model} & \textbf{Method} & F1 ($\uparrow$) & ROUGE-1 ($\uparrow$) & BLEU-1 ($\uparrow$) & F1 ($\uparrow$) & ROUGE-1 ($\uparrow$) & BLEU-1 ($\uparrow$) \\
  \midrule
  & Full Recompute            & 0.339 & 0.352 & 0.301 & 0.390 & 0.400 & 0.353 \\
  Qwen2.5-3B-Instruct & CacheBlend       & 0.178 & 0.188 & 0.144 & 0.286 & 0.294 & 0.265 \\
  & ProphetKV        & 0.125    & 0.130    & 0.103    & 0.329 & 0.336 & 0.296 \\
  & AgentKVShift (Ours) & \textbf{0.319} & \textbf{0.330} & \textbf{0.277} & \textbf{0.384} & \textbf{0.394} & \textbf{0.349} \\
  \midrule
  & Full Recompute            & 0.444    & 0.471    & 0.407    & 0.446 & 0.464 & 0.414 \\
  Qwen3-4B-Instruct & CacheBlend       & 0.360    & 0.383    & 0.327     & 0.360 & 0.370 & 0.334 \\
  & ProphetKV        & 0.315    & 0.320    & 0.285    & 0.385 & 0.398 & 0.360 \\
  & AgentKVShift (Ours) & \textbf{0.429}     & \textbf{0.463}     & \textbf{0.402}    & \textbf{0.435} & \textbf{0.450} & \textbf{0.401} \\
  \midrule
  & Full Recompute            & 0.305 & 0.346 & 0.295 & 0.509 & 0.524 & 0.456 \\                           
  Mistral-7B-Instruct-v0.3 & CacheBlend       & 0.208 & 0.212 & 0.171 & 0.290 & 0.294 & 0.173 \\  
  & ProphetKV      & 0.170    & 0.169    & 0.145    & 0.357 & 0.369 & 0.231 \\                           
  & AgentKVShift (Ours) & \textbf{0.282} & \textbf{0.294} & \textbf{0.243} & \textbf{0.491} &        
  \textbf{0.508} & \textbf{0.441} \\
  \bottomrule
  \end{tabular}%
  }
\end{table*}
\textbf{Agentic memory Systems.} We evaluate on two recent agentic memory systems that span the
dominant design choices in current memory systems. AMem~\cite{xu2025mem} adopts a Zettelkasten-style organization with 
per-turn notes, summaries, and tags, representing the note-based 
category. LiCoMemory~\cite{huang2025licomemory} organizes memory as a 
hierarchical graph with temporal indexing, representing the graph-based 
category.

\textbf{Benchmarks.} We evaluate on two complementary long-horizon memory benchmarks. LoCoMo~\cite{maharana2024evaluating} is the standard benchmark for long-term memory in \textit{dialogue}, consisting of $10$ long, multi-session conversations with QA testing multi-hop, temporal, and reasoning capabilities. AMA-Bench~\cite{zhao2026ama} is a recent benchmark designed for memory in \textit{agentic applications}, partitioning questions across six domains (Embodied AI, Game, OpenWorld QA, Software, Text2SQL, Web) and four memory capabilities (Recall, Causal Inference, State Updating, and State Abstraction). For our main evaluation on AMA-Bench using AMEM memory, we focus on the AMA-Bench-Recall , which directly probes the chunk-level retrieval fidelity that KV-reuse methods are designed to preserve. For the other three capabilities, which require integrating signals \emph{across} retrieved chunks and depend more on the agent's reasoning pipeline than on the cache itself. For completeness, we also discuss the evaluation on these additional capabilities in Appendix~\ref{app:full_amabench}. 

\textbf{Models.} We evaluate four open-source instruction-tuned LLMs 
spanning two families and a $10\times$ scale range. The 3B--7B models 
(Qwen2.5-3B-Instruct~\cite{qwen2.5}, Qwen3-4B-Instruct~\cite{qwen3technicalreport}, Mistral-7B-Instruct-v0.3~\cite{mistralai2024mistral7binstructv03}) are 
evaluated on LoCoMo. For AMA-Bench, we use Qwen3-32B~\cite{qwen3technicalreport} with the AMem 
agentic memory, as the benchmark's task complexity demands the reasoning 
capacity of larger-scale models.

\textbf{Baselines.} We compare against two state-of-the-art KV reuse methods. CacheBlend~\cite{cacheblend} selectively recomputes high-deviation tokens identified at a check layer, and ProphetKV~\cite{wang2026prophetkv} replaces divergence-based selection 
with a query-driven criterion. We also report the no-KV-reuse (full recompute) upper 
bound as a reference.

\textbf{Evaluation metrics.} We report token-level F1, ROUGE-1, and BLEU-1 against ground truth answers as our primary metrics. On AMA-Bench, we additionally report LLM-Judge accuracy, where GPT-4o \cite{openai2024gpt4o} serves as judge and produces a binary $0/1$ grade against the 
ground-truth answer.

\textbf{Implementation.} All experiments are run on a single NVIDIA A100 GPU, except for Qwen3-32B AMA-Bench evaluation and throughput profiling, which use a single NVIDIA H200 GPU due to memory and compute requirements. For all KV-reuse methods, we use a single check layer at $\ell_c{=}1$ and report results at recompute ratios specified per 
experiment.

\subsection{Long-Horizon Memory Benchmarks}
\label{sec:longhorizon_results}

\subsubsection{LoCoMo Evaluation}



\textbf{AgentKVShift recovers near full-recompute quality at low recompute budget.} As shown in Table~\ref{tab:results}, AgentKVShift recovers near full-recompute quality at $r{=}0.1$. We attribute this to our probe-guided mean-shift correcting the $\sim 90\%$ of unrefreshed tokens as well, extracting useful signal from cache entries that CacheBlend and ProphetKV leave stale.  On AMem with Qwen2.5-3B-Instruct, AgentKVShift reaches $0.319$ F1 against a Full Recompute reference of $0.339$ (a $5.9\%$ relative drop), whereas CacheBlend and ProphetKV drop by $47.5\%$ and $63.1\%$ under the same $10\%$ budget. Similarly, on LiCoMemory 
the gap reduces to $1.5\%$ on Qwen2.5-3B, $2.5\%$ on Qwen3-4B, and $3.5\%$ 
on Mistral-7B.

\textbf{Gains transfer across metrics.} The relative drop to 
Full Recompute stays within $1.5$--$6\%$ across configurations, and the 
same trend carries over to ROUGE-1 and BLEU-1, indicating that our 
correction recovers not only token-level overlap but also longer n-gram 
fidelity.

\subsubsection{AMA-Bench-Recall Evaluation}
\label{sec:amabench_results}

\textbf{AgentKVShift gains the most on domains where retrieval fidelity drives accuracy.} As shown in Table~\ref{tab:domain_typeA}, AgentKVShift achieves the highest F1 on five of six domains and the highest LLM-Judge accuracy on five of six, yielding the strongest weighted averages on both metrics ($0.284$ F1, $0.279$ accuracy). Importantly, the largest absolute gains are observed on OpenWorld~QA and Text2SQL, where the answer-bearing content is concentrated in a small subset of memory units~\cite{zhao2026ama}. We attribute this to our mean-shift correcting all reused KV entries within a memory unit rather than only the small fraction selected by token-deviation criteria, thus precisely matching what concentrated retrieval rewards. In comparison, CacheBlend lags AgentKVShift by $1.4$ F1 points and $3.9$ accuracy points on weighted average, while ProphetKV trails by $0.7$ F1 points and $4.3$ accuracy points.

\textbf{AgentKVShift effectively closes the gap to Full Recompute at $r{=}0.3$.} As shown in Table~\ref{tab:domain_typeA}, AgentKVShift recovers the majority of the Full Recompute F1 across all six domains, closing the gap to $0.012$ F1 points on weighted average ($0.284$ vs.\ $0.296$). In comparison, CacheBlend and ProphetKV leave a substantially larger gap, since neither corrects the chunk-level residual on unrefreshed tokens.

\textbf{Toward state-aware KV reuse.} Our extended results on the cross-chunk reasoning capabilities (Appendix~\ref{app:full_amabench}) show all KV-reuse methods leaving a wider gap to Full Recompute, consistent with these capabilities depending more on the agent's reasoning pipeline than on the cache itself. This points to KV-reuse mechanisms that are aware of the agent's reasoning state, as a promising direction for future research. Nonetheless, AgentKVShift remains the strongest KV-reuse strategy across all reasoning categories.

\begin{table*}[t]
  \centering
  \caption{F1 and binary LLM-Judge accuracy by domain on AMA-Bench-Recall using AMEM memory (Qwen3-32B, $r{=}0.3$, judged by GPT-4o \cite{openai2024gpt4o}). Values in parentheses denote the retention ratio relative to Full KV Recompute (no KV-reuse). Best KV-reuse result is \textbf{bold}, second best is \underline{underlined}. CB = CacheBlend, PKV = ProphetKV, \textbf{AgentKVShift} is our method.}
  \label{tab:domain_typeA}
  \footnotesize
  \setlength{\tabcolsep}{4pt}
  \renewcommand{\arraystretch}{1.05}
  \newcommand{\rv}[2]{#1{\scriptsize\,(#2\%)}}
  \newcommand{\rvb}[2]{\textbf{#1}{\scriptsize\,\textbf{(#2\%)}}}
  \newcommand{\rvs}[2]{\underline{#1}{\scriptsize\,\underline{(#2\%)}}}
  \resizebox{\textwidth}{!}{%
  \begin{tabular}{@{}l c ccc c ccc@{}}
    \toprule
    & \multicolumn{4}{c}{\textbf{F1($\uparrow$)}} & \multicolumn{4}{c}{\textbf{Accuracy($\uparrow$)}} \\
    \cmidrule(lr){2-5} \cmidrule(lr){6-9}
    \textbf{Domain}
      & \textbf{Full Rec.} & \textbf{CB} & \textbf{PKV} & \textbf{AgentKVShift}
      & \textbf{Full Rec. } & \textbf{CB} & \textbf{PKV} & \textbf{AgentKVShift} \\
    \midrule
    Embodied AI  & 0.392 & \rv{0.352}{89.8}  & \rvs{0.372}{94.9}  & \rvb{0.381}{97.2}  & 0.180 & \rvs{0.148}{82.2} & \rvb{0.164}{91.1} & \rvs{0.148}{82.2} \\
    Game         & 0.400 & \rvs{0.385}{96.3} & \rv{0.382}{95.5}   & \rvb{0.391}{97.8}  & 0.375 & \rvs{0.333}{88.8} & \rvs{0.333}{88.8} & \rvb{0.375}{100.0} \\
    OpenWorld QA & 0.336 & \rv{0.205}{61.0}  & \rvs{0.276}{82.1}  & \rvb{0.280}{83.3}  & 0.357 & \rv{0.214}{59.9}  & \rvs{0.245}{68.6} & \rvb{0.327}{91.6} \\
    Software     & 0.149 & \rv{0.150}{100.7} & \rvs{0.151}{101.3} & \rvb{0.156}{104.7} & 0.137 & \rvs{0.160}{116.8}& \rv{0.156}{113.9} & \rvb{0.170}{124.1} \\
    Text2SQL     & 0.342 & \rvs{0.324}{94.7} & \rv{0.318}{93.0}   & \rvb{0.335}{98.0}  & 0.372 & \rvs{0.287}{77.2} & \rv{0.269}{72.3}  & \rvb{0.336}{90.3} \\
    Web          & 0.283 & \rvb{0.273}{96.5} & \rvs{0.270}{95.4}  & \rv{0.260}{91.9}   & 0.304 & \rvs{0.264}{86.8} & \rv{0.248}{81.6}  & \rvb{0.296}{97.4} \\
    \midrule
    \textit{Wtd Avg} & 0.296 & \rv{0.270}{91.2} & \rvs{0.277}{93.6} & \rvb{0.284}{96.0} & 0.287 & \rvs{0.240}{83.6} & \rv{0.236}{82.2} & \rvb{0.279}{97.2} \\
    \bottomrule
  \end{tabular}%
  }
\end{table*}

\subsection{Effect of Recompute Ratio}
\label{sec:recompute_ratio}

Next, we extend our analysis to isolate the effect of the refresh budget. Specifically, we sweep $r \in \{0.05, 0.10, \ldots, 0.95\}$ on LoCoMo with Qwen2.5-3B-Instruct and report F1 in Figure~\ref{fig:sweep_f1}.

As shown, AgentKVShift reaches within $\sim$$2$ points F1 of Full Recompute at $r{=}0.1$ and stays close to the reference throughout. CacheBlend and ProphetKV, in contrast, drop sharply at low $r$ and only match the Full Recompute baseline near $r{=}0.45-0.55$.
\begin{wrapfigure}[12]{r}{0.4\textwidth}
  \vspace{-14pt}
  \centering
  \includegraphics[width=0.4\textwidth]{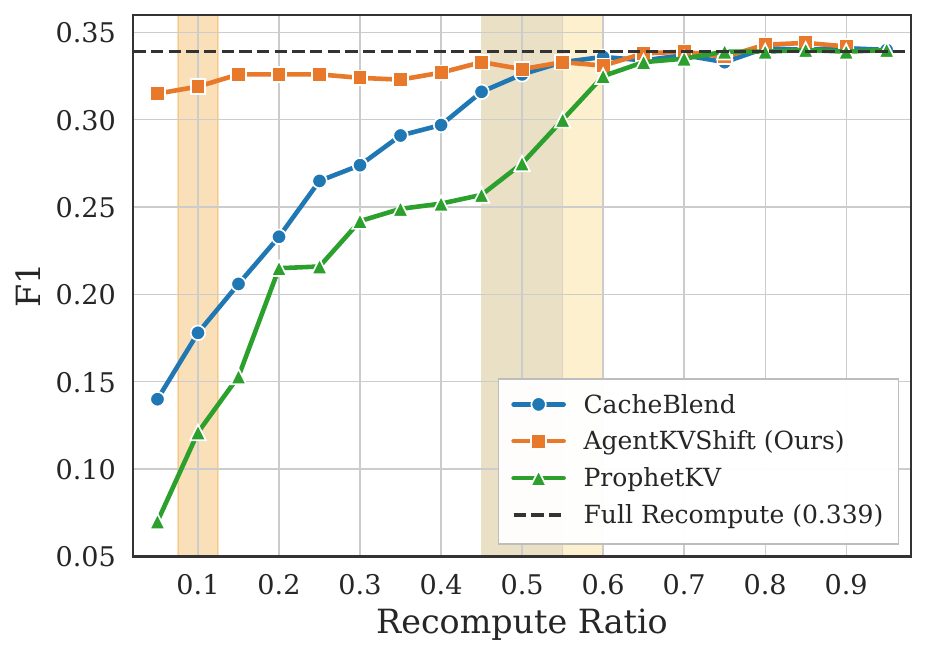}
  \caption{Effect of recompute ratio on F1.}
  \label{fig:sweep_f1}
  \vspace{-10pt}
\end{wrapfigure}

 The same ordering holds for ROUGE-1 and BLEU-1 as shown in Appendix~\ref{app:sweeps}. This reflects a difference in how each method uses the refresh budget.  CacheBlend and ProphetKV select \emph{which} tokens to recompute and leave the rest of the cache 
uncorrected, so their accuracy scales with the fraction of tokens refreshed. AgentKVShift instead corrects the reused keys themselves using a probe-estimated offset, so tokens that are not recomputed still contribute useful signal. As a result, AgentKVShift matches the quality that baselines achieve at $r{=}0.45$--$0.55$ while using $5\times$ fewer recomputations ($r{=}0.10$), translating directly into lower prefill latency at comparable output quality.
 \begin{table*}[t]                                                                                     
      \centering\scriptsize\setlength{\tabcolsep}{3.5pt}\renewcommand{\arraystretch}{0.95}              
      \caption{Speedup over Full KV recompute at $r{=}0.1$, by context length. Best per row in bold.  
  Qwen2.5-3B-Instruct and Qwen3-32B are profiled on a single A100-40GB and H200 respectively.}                                                                    
      \label{tab:speedup_per_toklen}                                                                    
      \begin{tabular}{@{}ll cccc cccc cccc cccc@{}}                                                     
      \toprule                                                                                          
       & & \multicolumn{4}{c}{\textit{2{,}048 tok/req}} & \multicolumn{4}{c}{\textit{4{,}096 tok/req}} &
      \multicolumn{4}{c}{\textit{8{,}192 tok/req}} & \multicolumn{4}{c}{\textit{16{,}384 tok/req}} \\   
      \cmidrule(lr){3-6}\cmidrule(lr){7-10}\cmidrule(lr){11-14}\cmidrule(lr){15-18}                     
       & & $B{=}1$ & 4 & 8 & 16 & $B{=}1$ & 4 & 8 & 16 & $B{=}1$ & 4 & 8 & 16 & $B{=}1$ & 4 & 8 & 16 \\ 
      \midrule                                                                                          
      \multirow{3}{*}{3B}  & CacheBlend & 1.19 & 3.53 & 3.73 & \textbf{4.01} & 2.55 & 3.41 & 3.64 &     
  \textbf{3.69} & 2.74 & 3.17 & 3.17 & \textbf{3.19} & 2.51 & 2.63 & \textbf{2.66} & OOM \\             
                           & ProphetKV  & 1.13 & 3.18 & 3.34 & \textbf{3.48} & 2.28 & 2.93 & 3.12 &
  \textbf{3.21} & 2.14 & 2.59 & 2.65 & \textbf{2.90} & 1.63 & 1.94 & \textbf{2.17} & OOM \\             
                           & AgentKVShift  & 0.94 & 2.97 & 3.31 & \textbf{3.55} & 1.88 & 3.06 & 3.24 &
  \textbf{3.33} & 2.47 & 2.84 & 2.87 & \textbf{2.89} & 2.37 & 2.38 & \textbf{2.39} & OOM \\             
      \midrule    
      \multirow{3}{*}{32B} & CacheBlend & 2.28 & 4.26 & 4.49 & \textbf{4.58} & 3.64 & 3.84 &    
  3.70 & \textbf{3.85} & 2.86 & \textbf{3.30} & 3.23 & 3.24 & \textbf{2.70} & 2.67 & 2.66 & 
  OOM \\                                                                                    
                       & ProphetKV  & 2.03 & 3.34 & 3.89 & \textbf{4.59} & 2.86 & 3.07 &    
  3.59 & \textbf{3.95} & 2.23 & 2.64 & 2.92 & \textbf{2.92} & 1.74 & \textbf{2.04} & 2.03 & 
  OOM \\
                       & AgentKVShift  & 1.53 & 3.62 & 3.95 & \textbf{4.13} & 3.09 & 3.45 &    
  \textbf{3.70} & 3.65 & 2.61 & 2.88 & 3.01 & \textbf{3.02} & \textbf{2.49} & 2.40 & 2.43 & 
  OOM \\             
      \bottomrule                                                                                       
      \end{tabular}                                                                                     
  \end{table*} 

\begin{figure*}[t]
    \centering
    \begin{subfigure}[t]{0.49\textwidth}
        \centering
        \includegraphics[height=4cm,keepaspectratio]{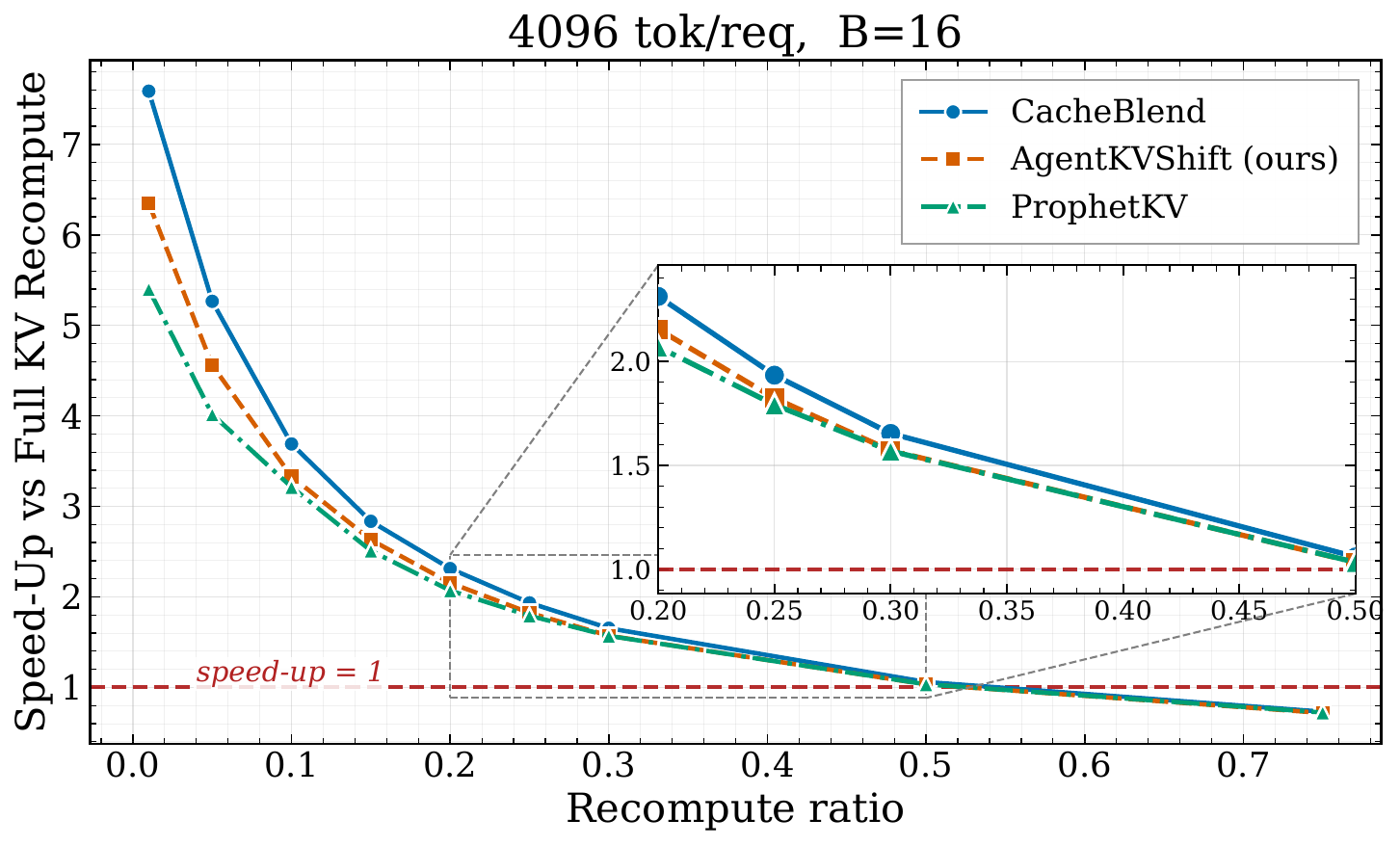}
        \caption{Qwen2.5-3B-Instruct}
        \label{fig:pareto_3b}
    \end{subfigure}\hfill
    \begin{subfigure}[t]{0.49\textwidth}
        \centering
        \includegraphics[height=4cm,keepaspectratio]{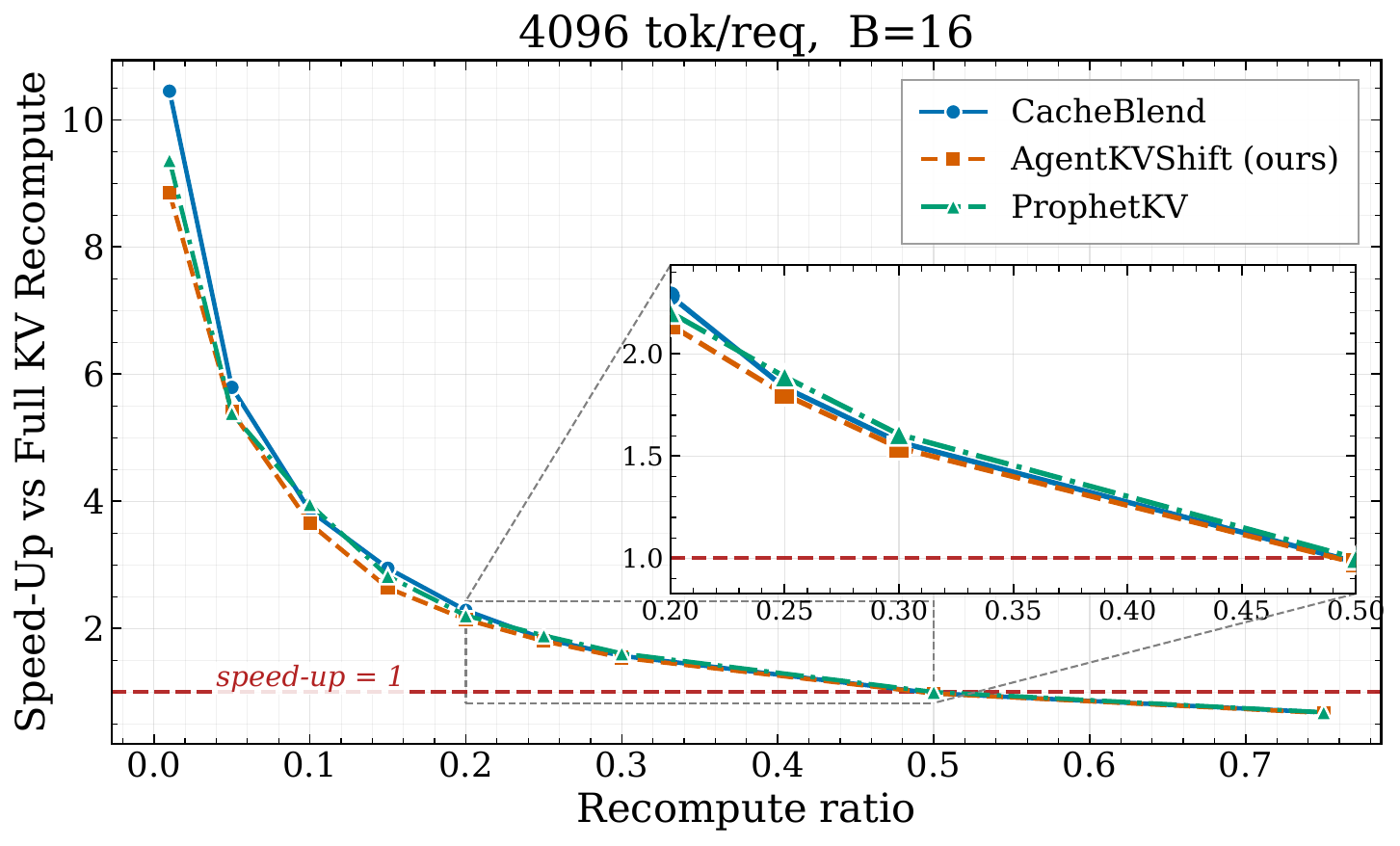}
        \caption{Qwen3-32B}
        \label{fig:pareto_14b}
    \end{subfigure}
    \caption{Speedup vs.\ prefill-latency Pareto sweep at 4{,}096 tok/req $\times$ $B{=}16$
    with $r \in \{0.01, \ldots, 0.75\}$}
    \label{fig:pareto_both}
\end{figure*}

\subsection{Throughput Analysis}
\label{sec:throughput}

LLM serving deployments span a wide range of operating points, where prefill latency depends jointly on context length and batch size, and the optimal KV-reuse strategy can shift as either dimension grows. Therefore, we profile prefill throughput across two model scales (Qwen2.5-3B-Instruct and Qwen3-32B), four context lengths ($2{,}048$--$16{,}384$ tok/req), and four batch sizes ($B \in \{1, 4, 8, 16\}$). Figure~\ref{fig:pareto_both} reports the speedup-vs-recompute-ratio sweep at $4{,}096$ tok/req with $B{=}16$ for $r \in \{0.01, \ldots, 0.75\}$. Whereas, Table~\ref{tab:speedup_per_toklen} fixes $r{=}0.1$ and sweeps the remaining context-length and batch-size configurations. We focus on $r{=}0.1$ and $r{=}0.3$, the operating points at which AgentKVShift preserves answer quality on LoCoMo and AMA-Bench respectively (Section~\ref{sec:longhorizon_results}).

\textbf{AgentKVShift matches the speedup curve of existing methods.} As shown in Figure~\ref{fig:pareto_both}, all three KV-reuse methods trace nearly identical speedup curves on both models. This is expected as at a fixed $r$ the recomputation work is the comparable across methods, and AgentKVShift's correction adds only a memory unit-level vector correction per layer, which is negligible compared to per-token recompute. As shown, AgentKVShift sustains a $\sim$$3.3\times$ prefill speedup over full KV recompute at $r{=}0.1$ on Qwen2.5-3B and a $\sim$$3.6\times$ speedup on Qwen3-32B, and roughly $1.5$--$1.6\times$ at $r{=}0.3$.


\textbf{Throughput gains scale with context length, batch size, and model size.} Next, we study these axes because they govern different deployment regimes. Long-context settings dominate prefill compute relative to per-request fixed costs, large-batch increases the GPU utilization of the correction stage, and larger models increase the per-token recompute cost. As shown in Table~\ref{tab:speedup_per_toklen}, at $r{=}0.1$, AgentKVShift's speedup over no-KV-reuse grows along all three axes. For Qwen2.5-3B at $4{,}096$ tok/req, the speedup increases from $1.88\times$ at $B{=}1$ to $3.33\times$ at $B{=}16$, while for Qwen3-32B the it increases from $3.09\times$ to $3.65\times$. At longer contexts, AgentKVShift sustains $2.4\times$--$3.0\times$ speedups across batch sizes on both models, demonstrating that the correction preserves benefits into the long-context regime where prefill cost dominates.

\subsection{Orthogonality to KV Quantization}
\label{sec:orthogonality}
Finally, we show that KV reuse is orthogonal to KV cache quantization. While KV reuse reduces prefill compute by avoiding redundant recomputation, KV quantization methods such as KIVI~\cite{liu2024kivi} and OTT~\cite{su2025accurate} reduce storage and data-movement cost by storing KV states at lower bit-widths, and accelerate inference through the use of appropriate low-precision compute kernels. To demonstrate this orthogonality, we pair CacheBlend, ProphetKV, 
AgentKVShift with five quantization configurations spanning FP8 (E4M3), KIVI 2/4-bit, and OTT 2/4-bit, and sweep $r \in \{0.1, 0.3\}$ on LoCoMo with Qwen2.5-3B-Instruct and using AMem memory. We further show that AgentKVShift consistently outperforms both CacheBlend and ProphetKV when combined with quantization, retaining substantially higher accuracy across all aggressive bit-width settings.

\textbf{AgentKVShift composes with quantization across all configurations.} As shown in Table~\ref{tab:kv_quant_f1}, AgentKVShift sustains the highest F1 under every quantization scheme at $r{=}0.1$ and $r{=}0.3$, while CacheBlend and ProphetKV degrade sharply under aggressive bit-widths (e.g., at KIVI 2-bit, $r{=}0.1$, both baselines collapse 
below $0.08$ F1 whereas AgentKVShift retains $0.203$).

\begin{wraptable}[10]{r}{0.50\textwidth}
\centering
\vspace{-0.15in}
\caption{F1 under KV cache quantization across recompute ratios $r$. CB = CacheBlend, PKV = ProphetKV, AKS = AgentKVShift}
\label{tab:kv_quant_f1}
\renewcommand{\arraystretch}{1}
\setlength{\tabcolsep}{5pt}
\scriptsize
\begin{tabular}{l|ccc|ccc}
\toprule
 & \multicolumn{3}{c|}{$\mathbf{r = 0.1}$} & \multicolumn{3}{c}{$r = 0.3$} \\
\cmidrule(lr){2-4} \cmidrule(lr){5-7}
\textbf{Method} & CB & PKV & AKS & CB & PKV & AKS \\
\midrule
No Quant   & 0.150 & 0.125 & \textbf{0.319} & 0.239 & 0.250 & \textbf{0.324} \\
FP8 (E4M3) & 0.146 & 0.006 & \textbf{0.321} & 0.240 & 0.256 & \textbf{0.323} \\
KIVI 2-bit & 0.076 & 0.012 & \textbf{0.203} & 0.114 & 0.109 & \textbf{0.184} \\
KIVI 4-bit & 0.190 & 0.102 & \textbf{0.277} & 0.277 & 0.239 & \textbf{0.314} \\
OTT 2-bit  & 0.080 & 0.038 & \textbf{0.202} & 0.113 & 0.110 & \textbf{0.166} \\
OTT 4-bit  & 0.196 & 0.100 & \textbf{0.317} & 0.237 & 0.243 & \textbf{0.323} \\
\bottomrule
\end{tabular}
\end{wraptable}

The mechanism follows directly from how the correction is computed: the probe estimate $\hat{\mu}_C^\ell = \tfrac{1}{b}\sum_{j \in S_C}(K_{j,\mathrm{fresh}}^\ell -K_{j,\mathrm{reuse}}^\ell)$ is the gap between fresh and reused keys, so when the cache is quantized, the estimate absorbs both the semantic drift from context changes and the average quantization noise introduced during storage. 
CacheBlend and ProphetKV, by contrast, retain the full quantization error in their reused keys. 

\section{Conclusion}
\label{sec:conclusion}

In this work, we introduced AgentKVShift, a training-free probe-guided KV residual correction method designed for agentic memory systems for long-horizon LLM agents. Across LoCoMo and AMA-Bench-Recall,
spanning both note- and graph-based memory systems and model scales from 3B to 32B parameters, AgentKVShift achieves near full-recompute quality, within 1.5--6\% relative F1 on LoCoMo and within 3\%
LLM-judge accuracy on AMA-Bench-Recall, while refreshing only 10\% and 30\% of the cache, respectively, with prefill speedups reaching 2--3.5$\times$ over no-KV-reuse within this
operating range. AgentKVShift also outperforms CacheBlend and ProphetKV across all tested configurations, including under aggressive 2- and 4-bit KV quantization where it retains over
2$\times$ their F1. By estimating a shared memory chunk level offset from a small probe set and applying a single weighted correction to every reused token, AgentKVShift turns the
refresh budget into useful signal across the entire chunk rather than leaving unrefreshed tokens stale. Overall, AgentKVShift offers a simple, yet effective correction method that couples residual structure
with retrieval-time KV reuse agentic memory workloads.

\section*{Acknowledgements}
This work has been funded in part by NSF, with award numbers \#2112665, \#2112167, \#2003279, \#2120019, \#2211386, \#2052809, \#1911095 and in part by PRISM and CoCoSys, centers in JUMP 2.0, an SRC program sponsored by DARPA. We also thank Modal for providing compute credits through their compute grant program, which supported the experiments in this work.
\bibliographystyle{plainnat}   
\bibliography{references}

\newpage 
\appendix

\definecolor{varred}{RGB}{139,0,30}
\newcommand{\var}[1]{{\color{varred}\{#1\}}}
\newtcolorbox{prompttemplate}[1]{
  enhanced, breakable,
  colback=white, colframe=black,
  boxrule=0.6pt, arc=4mm,
  left=10pt, right=10pt, top=12pt, bottom=10pt,
  fontupper=\normalsize,
  attach boxed title to top left={xshift=12pt, yshift=-\tcboxedtitleheight/2},
  boxed title style={
    colback=black, colframe=black,
    arc=2mm, boxrule=0pt,
    left=8pt, right=8pt, top=3pt, bottom=3pt,
  },
  coltitle=white, fonttitle=\bfseries,
  title={#1}
}

\section*{Contents}                           
\label{app:contents}

This appendix provides supplementary results, evaluation details, and
formal analysis supporting the main paper. It is organized into three
parts:

\begin{itemize}
  \item \textbf{Appendix~\ref{app:extended_results} — Extended Experimental Results.}
        Extended experimental setup and reproducibility details
        (\S\ref{app:setup}), the full AMA-Bench breakdown across all
        reasoning capabilities (\S\ref{app:full_amabench}), recompute
        ratio sweeps on ROUGE-1 and BLEU-1 (\S\ref{app:sweeps}), the
        LLM-as-a-Judge prompt used to produce the accuracy scores
        (\S\ref{app:llm_judge}), and results under six KV-cache
        quantization regimes (\S\ref{app:kv_quant}).
 \item \textbf{Appendix~\ref{app:theory} — Theoretical Analysis.}
      Discussions and proofs of the deterministic attention
      perturbation bound (\S\ref{lem1}), the error bounds
      before and after probe-guided correction (\S\ref{prop1},
      \S\ref{prop2}), and the sufficient condition under which
      correction strictly tightens the bound (\S\ref{prop3}).
\item \textbf{Appendix~\ref{emp_val} — Empirical Validation
      of the Modeling Assumptions.} Diagnostics that supports the
      residual decomposition $r_i=\mu+\xi_i$ via mean-explained
      energy and top-singular-value suppression
      (\S\ref{mean_explained}), and the sub-Gaussian fluctuation
      assumption via probe-mean concentration and random-projection
      QQ plots (\S\ref{fluctuation_validation}).
\end{itemize}

\section{Extended Experimental Results}
\label{app:extended_results}

\subsection{Extended Experimental Set-Up and Reproducibility}
\label{app:setup}

\paragraph{Models:}
We evaluate four open-source instruction-tuned models spanning two families and a
$\sim$10$\times$ parameter range. All checkpoints are obtained from the HuggingFace Hub
and listed in Table~\ref{tab:app-models}.

\begin{table}[h]
\centering
\small
\caption{Models}
\label{tab:app-models}
\begin{tabular}{ll}
\toprule
\textbf{Model} & \textbf{HF Identifier} \\
\midrule
Qwen2.5-3B-Instruct       & \texttt{Qwen/Qwen2.5-3B-Instruct} \\
Qwen3-4B-Instruct         & \texttt{Qwen/Qwen3-4B-Instruct-2507} \\
Mistral-7B-Instruct-v0.3  & \texttt{mistralai/Mistral-7B-Instruct-v0.3} \\
Qwen3-32B                 & \texttt{Qwen/Qwen3-32B} \\
\bottomrule
\end{tabular}
\end{table}

\paragraph{Agentic Memory Systems:}
We evaluate against two recent memory systems:
\begin{itemize}
  \item \textbf{A-MEM}~\cite{xu2025mem} --
    Source: \url{https://github.com/agiresearch/a-mem}
  \item \textbf{LiCoMemory}~\cite{huang2025licomemory} --
    Source: \url{https://github.com/EverM0re/LiCoMemory}
\end{itemize}
We use each system's default ingestion pipeline and top-$k$ retrieval with $k=10$, and
run all KV-reuse methods on top of the created memory units without modification to the memory system itself.

\paragraph{Benchmarks:}
\begin{itemize}
  \item \textbf{LoCoMo}~\cite{maharana2024evaluating} -- 
    Source: \url{https://github.com/snap-research/locomo}
  \item \textbf{AMA-Bench}~\cite{zhao2026ama} -- 
    Source: \url{https://huggingface.co/datasets/AMA-bench/AMA-bench}
\end{itemize}

\paragraph{Hardware:}
LoCoMo accuracy and throughput experiments with the 3B--7B models run on a single
NVIDIA A100 80GB. AMA-Bench accuracy experiments and throughput profiling with
Qwen3-32B run on a single NVIDIA H200 due to memory requirements.

\subsection{Extended AMA-Bench Discussion}
\label{app:full_amabench}
\begin{table*}[!htbp]
  \centering
  \caption{F1 and LLM-judge accuracy by domain on AMA-Bench
  (Qwen3-32B, $r{=}0.3$, judged by GPT-4o). Values in parentheses
  denote the retention ratio relative to Full Rec. Best KV-reuse result
  is \textbf{bold}, second best is \underline{underlined}.
  CB = CacheBlend, PKV = ProphetKV, \textbf{AgentKVShift} is our method.}
  \label{tab:domain_f1_judge}
  \footnotesize
  \setlength{\tabcolsep}{4pt}
  \renewcommand{\arraystretch}{1.05}
  \newcommand{\rv}[2]{#1{\scriptsize\,(#2\%)}}
  \newcommand{\rvb}[2]{\textbf{#1}{\scriptsize\,\textbf{(#2\%)}}}
  \newcommand{\rvs}[2]{\underline{#1}{\scriptsize\,\underline{(#2\%)}}}
  \resizebox{\textwidth}{!}{%
  \begin{tabular}{@{}l c ccc c ccc@{}}
    \toprule
    & \multicolumn{4}{c}{\textbf{F1($\uparrow$)}} & \multicolumn{4}{c}{\textbf{Accuracy($\uparrow$)}} \\
    \cmidrule(lr){2-5} \cmidrule(lr){6-9}
    \textbf{Domain}
      & \textbf{Full Rec.} & \textbf{CB} & \textbf{PKV} & \textbf{AgentKVShift}
      & \textbf{Full Rec.} & \textbf{CB} & \textbf{PKV} & \textbf{AgentKVShift} \\
    \midrule
    Embodied AI  & 0.418 & \rv{0.334}{79.9}  & \rvs{0.354}{84.7}  & \rvb{0.376}{89.9}  & 0.122 & \rv{0.064}{52.5}  & \rvs{0.067}{54.9}  & \rvb{0.081}{66.4} \\
    Game         & 0.363 & \rv{0.316}{87.1}  & \rvs{0.337}{92.8}  & \rvb{0.338}{93.1}  & 0.367 & \rv{0.281}{76.6}  & \rvs{0.308}{83.9}  & \rvb{0.322}{87.7} \\
    OpenWorld QA & 0.311 & \rv{0.127}{40.8}  & \rvs{0.196}{63.0}  & \rvb{0.200}{64.3}  & 0.386 & \rv{0.147}{38.1}  & \rvs{0.158}{40.9}  & \rvb{0.278}{72.0} \\
    Software     & 0.175 & \rv{0.176}{100.6} & \rvs{0.179}{102.3} & \rvb{0.183}{104.6} & 0.141 & \rvs{0.162}{114.9} & \rvb{0.167}{118.4} & \rv{0.155}{109.9} \\
    Text2SQL     & 0.298 & \rv{0.240}{80.5}  & \rvs{0.260}{87.2}  & \rvb{0.264}{88.6}  & 0.222 & \rv{0.137}{61.7}  & \rvs{0.154}{69.4}  & \rvb{0.183}{82.4} \\
    Web          & 0.277 & \rv{0.199}{71.8}  & \rvb{0.245}{88.4}  & \rvs{0.231}{83.4}  & 0.349 & \rv{0.218}{62.5}  & \rvs{0.263}{75.4}  & \rvb{0.288}{82.5} \\
    \midrule
    \textit{Wtd.\ Avg.} & 0.302 & \rv{0.231}{76.5} & \rvs{0.259}{85.8} & \rvb{0.263}{87.1} & 0.257 & \rv{0.165}{64.2} & \rvs{0.183}{71.2} & \rvb{0.213}{82.9} \\
    \bottomrule
  \end{tabular}%
  }
\end{table*}

\begin{table*}[!htbp]
  \centering
  \caption{F1 and LLM-judge accuracy by question type on AMA-Bench
  (Qwen3-32B, $r{=}0.3$, judged by GPT-4o). Values in parentheses
  denote the retention ratio relative to Full Rec. Best KV-reuse result
  is \textbf{bold}, second best is \underline{underlined}.
  CB = CacheBlend, PKV = ProphetKV, \textbf{AgentKVShift} is our method.}
  \label{tab:type_f1_judge}
  \footnotesize
  \setlength{\tabcolsep}{4pt}
  \renewcommand{\arraystretch}{1.05}
  \newcommand{\rv}[2]{#1{\scriptsize\,(#2\%)}}
  \newcommand{\rvb}[2]{\textbf{#1}{\scriptsize\,\textbf{(#2\%)}}}
  \newcommand{\rvs}[2]{\underline{#1}{\scriptsize\,\underline{(#2\%)}}}
  \resizebox{\textwidth}{!}{%
  \begin{tabular}{@{}l c ccc c ccc@{}}
    \toprule
    & \multicolumn{4}{c}{\textbf{F1($\uparrow$)}} & \multicolumn{4}{c}{\textbf{Accuracy($\uparrow$)}} \\
    \cmidrule(lr){2-5} \cmidrule(lr){6-9}
    \textbf{Capability}
      & \textbf{Full Rec.} & \textbf{CB} & \textbf{PKV} & \textbf{AgentKVShift}
      & \textbf{Full Rec.} & \textbf{CB} & \textbf{PKV} & \textbf{AgentKVShift} \\
    \midrule
    A.\ Recall            & 0.296 & \rv{0.270}{91.2}  & \rvs{0.277}{93.6}  & \rvb{0.284}{95.9}  & 0.287 & \rvs{0.240}{83.6} & \rv{0.236}{82.2}  & \rvb{0.279}{97.2} \\
    B.\ Causal Inference  & 0.284 & \rv{0.213}{75.0}  & \rvs{0.244}{85.9}  & \rvb{0.252}{88.7}  & 0.277 & \rv{0.146}{52.7}  & \rvs{0.191}{69.0} & \rvb{0.218}{78.7} \\
    C.\ State Updating    & 0.320 & \rv{0.211}{65.9}  & \rvs{0.254}{79.4}  & \rvb{0.263}{82.2}  & 0.252 & \rv{0.119}{47.2}  & \rvs{0.147}{58.3} & \rvb{0.182}{72.2} \\
    D.\ State Abstraction & 0.311 & \rv{0.211}{67.8}  & \rvb{0.254}{81.7}  & \rvs{0.235}{75.6}  & 0.176 & \rv{0.114}{64.8}  & \rvb{0.118}{67.0} & \rvb{0.118}{67.0} \\
    \midrule
    \textit{Wtd.\ Avg.}   & 0.302 & \rv{0.231}{76.5} & \rvs{0.259}{85.8} & \rvb{0.263}{87.1}  & 0.257 & \rv{0.165}{64.2}  & \rvs{0.183}{71.2} & \rvb{0.213}{82.9} \\
    \bottomrule
  \end{tabular}%
  }
\end{table*}

The main text (\S\ref{sec:amabench_results}) focuses on the Recall subset of
AMA-Bench because it cleanly isolates the per-chunk retrieval signal that KV
reuse is designed to preserve. In this appendix, we report the full breakdown
across all four reasoning capabilities, Recall (A), Causal Inference (B),
State Updating (C), and State Abstraction (D), reported both by domain
(Table~\ref{tab:domain_f1_judge}) and by capability
(Table~\ref{tab:type_f1_judge}). Capabilities B--D require integrating
evidence \emph{across} retrieved chunks, and therefore stress KV reuse
beyond chunk-local fidelity.

Across both views, AgentKVShift consistently achieves the best F1 and
LLM-judge accuracy among KV-reuse baselines on the weighted average,
narrowing the gap to full recomputation. The advantage is most pronounced
on capabilities that require cross-chunk integration (Causal Inference and
State Updating), where competing reuse methods (CB, PKV) degrade sharply
relative to the full-recompute reference.

\subsection{Recompute Ratio Sweeps}
\label{app:sweeps}

To complement the F1 results in the main text, we also report ROUGE-1
(Figure~\ref{fig:sweep_rouge1}) and BLEU-1 (Figure~\ref{fig:sweep_bleu1})
across recompute ratios $r$. Both surface-form metrics track the F1 trend
closely: our method approaches the no KV reuse baseline ($r{=}1.0$)
at substantially lower recompute budgets than competing KV-reuse
baselines.

\begin{figure}[H]
  \centering
  \begin{subfigure}[t]{0.48\linewidth}
    \includegraphics[width=\linewidth]{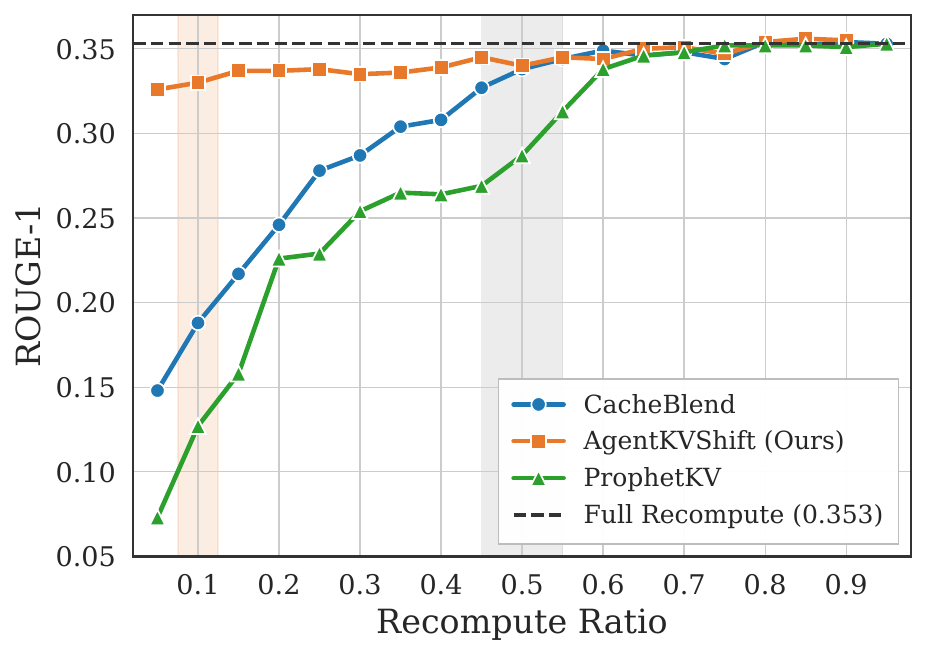}
    \caption{ROUGE-1}
    \label{fig:sweep_rouge1}
  \end{subfigure}\hfill
  \begin{subfigure}[t]{0.48\linewidth}
    \includegraphics[width=\linewidth]{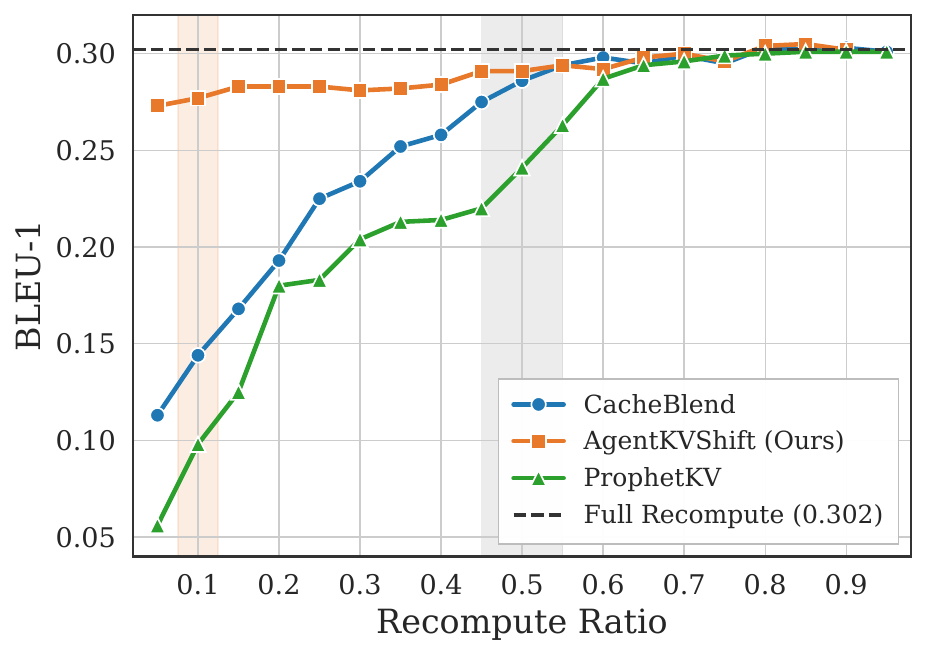}
    \caption{BLEU-1}
    \label{fig:sweep_bleu1}
  \end{subfigure}
  \caption{Recompute-ratio sweep on ROUGE-1 (left) and BLEU-1 (right). The dashed line denotes the no KV reuse baseline.}
  \label{fig:sweep_rouge_bleu}
\end{figure}

\subsection{LLM-as-a-Judge Evaluation Protocol}
\label{app:llm_judge}

For the accuracy column reported in Tables~\ref{tab:domain_f1_judge} and
\ref{tab:type_f1_judge}, we use GPT-4o as an LLM judge, prompted to
compare each model response against a reference answer and emit a
single binary decision. The exact prompt is shown below; placeholders
are highlighted in red.

\begin{prompttemplate}{Prompt template for LLM-as-a-Judge (single-reference, binary).}
\textbf{[Instruction]} Please act as an impartial judge and evaluate the
quality of the response provided by an AI assistant to the user question
below. Your evaluation should consider correctness. You will be given a
reference (golden) answer and the assistant's answer. Begin by comparing
the assistant's answer with the reference answer. Identify and correct
any mistakes. Be as objective as possible.

\smallskip
Your task is to determine if the predicted answer is correct based on:
(1) factual correctness compared to the reference, (2) completeness of
the answer, and (3) relevance to the question.

\smallskip
\textbf{[Question]}\quad \var{question}

\smallskip
\textbf{[Reference Answer]}\quad \var{golden\_answer}

\smallskip
\textbf{[The Start of Assistant's Answer]}\quad \var{response}\quad
\textbf{[The End of Assistant's Answer]}

\smallskip
Is the predicted answer correct? Respond with ONLY ``yes'' or ``no''.
Do not include any explanation.
\end{prompttemplate}

\subsection{Extended Results under KV Cache Quantization}
\label{app:kv_quant}

To assess robustness when KV reuse is composed with KV cache
quantization, Table~\ref{tab:kv-quant-full} reports F1, ROUGE-1, and
BLEU-1 across recompute ratios $r\in\{0.1,0.2,0.3,0.4,0.5\}$ for six
quantization regimes: no quantization, FP8 (E4M3), KIVI 2- and 4-bit,
and OTT 2- and 4-bit. We compare three reuse pipelines: CacheBlend
(CB), our AgentKVShift correction method, and ProphetKV (PKV).

Our method dominates across nearly every (quantization, ratio) cell,
with two consistent patterns. First, the gap between our method and
the next-best baseline is largest at low recompute ratios
($r=0.1, 0.2$), where competing methods degrade sharply but ours
remains close to its full-recompute ceiling. Second, our method is
notably more robust under aggressive quantization: under KIVI 2-bit
and OTT 2-bit at $r=0.1$, our F1 ($\approx 0.20$) more than doubles
that of CacheBlend and is an order of magnitude above ProphetKV,
indicating that probe-guided mean correction substantially mitigates
the compounding errors introduced when reuse and low-bit quantization
are stacked. The only cells where a baseline narrowly edges ours are
at high recompute ratios ($r=0.5$) under 2-/4-bit KIVI, where the
extra fresh KV tokens dominate the error budget.

\begin{table}[!htbp]
\centering
\small
\setlength{\tabcolsep}{4pt}
\renewcommand{\arraystretch}{0.82}
\caption{F1, ROUGE-1, and BLEU-1 across quantization methods,
recompute ratios, and reuse pipelines on the LoCoMo benchmark.
Bold denotes the best result per row.}
\label{tab:kv-quant-full}
\begin{tabular}{ll ccc ccc ccc}
\toprule
& & \multicolumn{3}{c}{\textbf{CacheBlend}} & \multicolumn{3}{c}{\textbf{ProphetKV}} & \multicolumn{3}{c}{\textbf{AgentKVShift}} \\
\cmidrule(lr){3-5}\cmidrule(lr){6-8}\cmidrule(lr){9-11}
\textbf{Quant.} & \textbf{Ratio} & F1 & R-1 & B-1 & F1 & R-1 & B-1 & F1 & R-1 & B-1 \\
\midrule
\multirow{5}{*}{No Quant} & 0.1 & 0.150 & 0.159 & 0.121 & 0.125 & 0.130 & 0.103 & \textbf{0.319} & \textbf{0.330} & \textbf{0.277} \\
 & 0.2 & 0.198 & 0.210 & 0.162 & 0.215 & 0.226 & 0.180 & \textbf{0.326} & \textbf{0.337} & \textbf{0.283} \\
 & 0.3 & 0.239 & 0.252 & 0.201 & 0.250 & 0.263 & 0.213 & \textbf{0.324} & \textbf{0.335} & \textbf{0.281} \\
 & 0.4 & 0.280 & 0.291 & 0.241 & 0.283 & 0.296 & 0.246 & \textbf{0.327} & \textbf{0.339} & \textbf{0.284} \\
 & 0.5 & 0.326 & 0.338 & 0.286 & 0.305 & 0.316 & 0.271 & \textbf{0.329} & \textbf{0.340} & \textbf{0.291} \\
\midrule
\multirow{5}{*}{FP8 (E4M3)} & 0.1 & 0.146 & 0.154 & 0.117 & 0.006 & 0.006 & 0.004 & \textbf{0.321} & \textbf{0.333} & \textbf{0.279} \\
 & 0.2 & 0.196 & 0.208 & 0.160 & 0.199 & 0.207 & 0.165 & \textbf{0.328} & \textbf{0.339} & \textbf{0.284} \\
 & 0.3 & 0.240 & 0.253 & 0.202 & 0.256 & 0.269 & 0.220 & \textbf{0.323} & \textbf{0.334} & \textbf{0.279} \\
 & 0.4 & 0.275 & 0.287 & 0.237 & 0.275 & 0.287 & 0.240 & \textbf{0.324} & \textbf{0.336} & \textbf{0.281} \\
 & 0.5 & 0.325 & 0.336 & 0.285 & 0.297 & 0.308 & 0.261 & \textbf{0.328} & \textbf{0.340} & \textbf{0.290} \\
\midrule
\multirow{5}{*}{KIVI 2-bit} & 0.1 & 0.076 & 0.078 & 0.055 & 0.012 & 0.014 & 0.009 & \textbf{0.203} & \textbf{0.210} & \textbf{0.165} \\
 & 0.2 & 0.094 & 0.099 & 0.072 & 0.072 & 0.076 & 0.055 & \textbf{0.200} & \textbf{0.207} & \textbf{0.166} \\
 & 0.3 & 0.114 & 0.119 & 0.091 & 0.109 & 0.114 & 0.084 & \textbf{0.184} & \textbf{0.190} & \textbf{0.154} \\
 & 0.4 & 0.144 & 0.150 & 0.117 & 0.137 & 0.143 & 0.108 & \textbf{0.188} & \textbf{0.193} & \textbf{0.156} \\
 & 0.5 & \textbf{0.255} & \textbf{0.267} & \textbf{0.217} & 0.169 & 0.175 & 0.138 & 0.244 & 0.254 & 0.206 \\
\midrule
\multirow{5}{*}{KIVI 4-bit} & 0.1 & 0.190 & 0.196 & 0.154 & 0.102 & 0.108 & 0.084 & \textbf{0.277} & \textbf{0.287} & \textbf{0.237} \\
 & 0.2 & 0.234 & 0.247 & 0.193 & 0.182 & 0.191 & 0.150 & \textbf{0.309} & \textbf{0.321} & \textbf{0.266} \\
 & 0.3 & 0.277 & 0.289 & 0.237 & 0.239 & 0.251 & 0.200 & \textbf{0.314} & \textbf{0.327} & \textbf{0.273} \\
 & 0.4 & 0.299 & 0.310 & 0.260 & 0.252 & 0.265 & 0.216 & \textbf{0.326} & \textbf{0.338} & \textbf{0.284} \\
 & 0.5 & \textbf{0.331} & \textbf{0.343} & 0.292 & 0.275 & 0.286 & 0.239 & 0.330 & 0.341 & \textbf{0.293} \\
\midrule
\multirow{5}{*}{OTT 2-bit} & 0.1 & 0.080 & 0.082 & 0.058 & 0.038 & 0.040 & 0.030 & \textbf{0.202} & \textbf{0.209} & \textbf{0.165} \\
 & 0.2 & 0.103 & 0.107 & 0.078 & 0.079 & 0.084 & 0.061 & \textbf{0.200} & \textbf{0.206} & \textbf{0.165} \\
 & 0.3 & 0.113 & 0.118 & 0.090 & 0.110 & 0.116 & 0.085 & \textbf{0.166} & \textbf{0.171} & \textbf{0.137} \\
 & 0.4 & 0.145 & 0.150 & 0.118 & 0.138 & 0.146 & 0.108 & \textbf{0.180} & \textbf{0.185} & \textbf{0.150} \\
 & 0.5 & 0.257 & 0.267 & 0.220 & 0.168 & 0.176 & 0.136 & \textbf{0.264} & \textbf{0.276} & \textbf{0.228} \\
\midrule
\multirow{5}{*}{OTT 4-bit} & 0.1 & 0.196 & 0.204 & 0.160 & 0.100 & 0.106 & 0.081 & \textbf{0.317} & \textbf{0.329} & \textbf{0.275} \\
 & 0.2 & 0.202 & 0.214 & 0.164 & 0.182 & 0.191 & 0.150 & \textbf{0.325} & \textbf{0.337} & \textbf{0.281} \\
 & 0.3 & 0.237 & 0.249 & 0.199 & 0.243 & 0.255 & 0.204 & \textbf{0.323} & \textbf{0.334} & \textbf{0.279} \\
 & 0.4 & 0.280 & 0.291 & 0.241 & 0.258 & 0.270 & 0.220 & \textbf{0.324} & \textbf{0.336} & \textbf{0.281} \\
 & 0.5 & 0.327 & 0.338 & 0.287 & 0.274 & 0.284 & 0.238 & \textbf{0.328} & \textbf{0.339} & \textbf{0.291} \\
\bottomrule
\end{tabular}
\end{table}

\section{Theoretical Analysis}
\label{app:theory}

This section provides the formal statements and proofs supporting the
analysis in \S\ref{sec:theory}. We first establish a deterministic
attention-perturbation bound (\S\ref{lem1}). We then specialize
it to the KV-reuse setting both before (\S\ref{prop1}) and after
(\S\ref{prop2}) probe-guided mean correction, and finally give a
sufficient condition under which correction yields a strictly tighter
high-probability error bound (\S\ref{prop3}).

\vspace{2in}
\subsection{Lemma 1: Deterministic Attention Perturbation Bound}
\label{lem1}
\begin{lem}[Deterministic attention perturbation bound]
Let $q \in \mathbb{R}^{d}$ be a fixed query, and let
$$
k_i^\star,\; v_i^\star \in \mathbb{R}^{d}, \qquad i=1,\dots,n
$$
denote the fresh keys and values for a single attention head. Define
$$
z_i^\star := \frac{q^\top k_i^\star}{\sqrt d}
\qquad
\alpha_i^\star := \frac{\exp(z_i^\star)}{\sum_{m=1}^n \exp(z_m^\star)},
\qquad
y^\star := \sum_{i=1}^n \alpha_i^\star v_i^\star
$$
For arbitrary perturbed keys and values $k_i, v_i \in \mathbb{R}^d$, define
$$
z_i := \frac{q^\top k_i}{\sqrt d}
\qquad
\alpha_i := \frac{\exp(z_i)}{\sum_{m=1}^n \exp(z_m)},
\qquad
y := \sum_{i=1}^n \alpha_i v_i
$$
Assume that $\|v_i^\star\|_2 \le V_{\max}$ for all $i$. Then
$$
\|y-y^\star\|_2
\le
\frac{2V_{\max}\|q\|_2}{2\sqrt d}\max_{1\le i\le n}\|k_i-k_i^\star\|_2
+
\max_{1\le i\le n}\|v_i-v_i^\star\|_2
$$
\end{lem}

\textit{Proof. } 
Write out the gap between clear attention output and perturbed attention output:
$$
y-y^\star
=
\sum_{i=1}^n \alpha_i v_i-\sum_{i=1}^n \alpha_i^\star v_i^\star
=
\sum_{i=1}^n (\alpha_i-\alpha_i^\star)v_i^\star
+
\sum_{i=1}^n \alpha_i (v_i-v_i^\star)
$$
Hence, by the triangle inequality,
$$
\|y-y^\star\|_2
\le
\left\|\sum_{i=1}^n (\alpha_i-\alpha_i^\star)v_i^\star\right\|_2
+
\left\|\sum_{i=1}^n \alpha_i (v_i-v_i^\star)\right\|_2
$$
For the first term, using $\|v_i^\star\|_2 \le V_{\max}$,
$$
\left\|\sum_{i=1}^n (\alpha_i-\alpha_i^\star)v_i^\star\right\|_2
\le
\sum_{i=1}^n |\alpha_i-\alpha_i^\star|\,\|v_i^\star\|_2
\le
V_{\max}\|\alpha-\alpha^\star\|_1
$$
For the second term, since $\alpha_i \ge 0$ and $\sum_i \alpha_i=1$,
$$
\left\|\sum_{i=1}^n \alpha_i (v_i-v_i^\star)\right\|_2
\le
\sum_{i=1}^n \alpha_i \|v_i-v_i^\star\|_2
\le
\max_{1\le i\le n}\|v_i-v_i^\star\|_2
$$
Therefore,
$$
\|y-y^\star\|_2
\le
V_{\max}\|\alpha-\alpha^\star\|_1
+
\max_{1\le i\le n}\|v_i-v_i^\star\|_2
$$
Next, define the logit perturbations
$$
\delta_i := z_i-z_i^\star
=
\frac{q^\top (k_i-k_i^\star)}{\sqrt d}.
$$
Use the fact that softmax is $1/2$-Lipschitz~\citep{nair2025softmax},
$$
\|\alpha-\alpha^\star\|_1 \le \frac12 \|\delta\|_\infty
$$
Moreover,
$$
|\delta_i|
=
\left|\frac{q^\top (k_i-k_i^\star)}{\sqrt d}\right|
\le
\frac{\|q\|_2\,\|k_i-k_i^\star\|_2}{\sqrt d}
$$
so
$$
\|\delta\|_\infty
\le
\frac{\|q\|_2}{\sqrt d}
\max_{1\le i\le n}\|k_i-k_i^\star\|_2
$$
Combining the above inequalities yields
$$
\|y-y^\star\|_2
\le
\frac{V_{\max}\|q\|_2}{2\sqrt d}\max_{1\le i\le n}\|k_i-k_i^\star\|_2
+
\max_{1\le i\le n}\|v_i-v_i^\star\|_2
$$

\begin{flushright} $\blacksquare$ \end{flushright}

\subsection{Proposition 1: Attention Error Bound Before Correction}
\label{prop1}
\begin{prop}[Attention Error Bound Before Correction]
Again, let $q \in \mathbb{R}^d$ be a fixed query, and define the offsets:
$$
r_i^K := k_i^\star-k_i^{\mathrm{reuse}},
\qquad
r_i^V := v_i^\star-v_i^{\mathrm{reuse}}
$$
and assume the decomposition
$$
r_i^K = \mu_K + \xi_i^K,
\qquad
r_i^V = \mu_V + \xi_i^V
$$
where $\mu_K,\mu_V \in \mathbb{R}^d$ are fixed chunk-level offsets and $\xi_i^K,\xi_i^V$ are mean-zero fluctuations. Define
$$
z_i^{\mathrm{reuse}} := \frac{q^\top k_i^{\mathrm{reuse}}}{\sqrt d},
\qquad
\alpha_i^{\mathrm{reuse}} := \frac{\exp(z_i^{\mathrm{reuse}})}{\sum_{m=1}^n \exp(z_m^{\mathrm{reuse}})},
\qquad
y^{\mathrm{reuse}} := \sum_{i=1}^n \alpha_i^{\mathrm{reuse}} v_i^{\mathrm{reuse}}
$$

We assume boundedness for value vectors $(\|v_i^\star\|_2 \le V_{\max}) \forall i$ and centered token-wise residual fluctuations are sub-Gaussian:
$$
\Pr(\|\xi_i^K\|_2 \ge t) \le 2e^{-t^2/(2\sigma_K^2)},
\qquad
\Pr(\|\xi_i^V\|_2 \ge t) \le 2e^{-t^2/(2\sigma_V^2)}
$$

Then for any $0<\delta<1$, with probability at least $1-\delta$,
$$
\|y^{\mathrm{reuse}}-y^\star\|_2
\le
\frac{V_{\max}\|q\|_2}{2\sqrt d}
\left(\|\mu_K\|_2+\sigma_K\sqrt{2\log\frac{4n}{\delta}}\right)
+
\left(\|\mu_V\|_2+\sigma_V\sqrt{2\log\frac{4n}{\delta}}\right)
$$
\end{prop}

\textit{Proof. }
Apply Lemma~1 with
$$
k_i = k_i^{\mathrm{reuse}},
\qquad
v_i = v_i^{\mathrm{reuse}}
$$
Since
$$
k_i-k_i^\star = -r_i^K,
\qquad
v_i-v_i^\star = -r_i^V
$$
Lemma~1 gives
$$
\|y^{\mathrm{reuse}}-y^\star\|_2
\le
\frac{V_{\max}\|q\|_2}{2\sqrt d}\max_{1\le i\le n}\|r_i^K\|_2
+
\max_{1\le i\le n}\|r_i^V\|_2
$$
Using
$$
r_i^K=\mu_K+\xi_i^K,
\qquad
r_i^V=\mu_V+\xi_i^V
$$
The triangle inequality yields
$$
\max_{1\le i\le n}\|r_i^K\|_2
\le
\|\mu_K\|_2+\max_{1\le i\le n}\|\xi_i^K\|_2,
$$
$$
\max_{1\le i\le n}\|r_i^V\|_2
\le
\|\mu_V\|_2+\max_{1\le i\le n}\|\xi_i^V\|_2
$$
By the sub-Gaussian tail bound and a union bound,
$$
\Pr\!\left(\max_{1\le i\le n}\|\xi_i^K\|_2 \ge t\right)
\le
\sum_{i=1}^n \Pr(\|\xi_i^K\|_2 \ge t)
\le
2n e^{-t^2/(2\sigma_K^2)}
$$
Setting the right-hand side equal to $\delta/2$ gives
$$
t=\sigma_K\sqrt{2\log\frac{4n}{\delta}}
$$
so with probability at least $1-\delta/2$,
$$
\max_{1\le i\le n}\|\xi_i^K\|_2
\le
\sigma_K\sqrt{2\log\frac{4n}{\delta}}
$$
Similarly, with probability at least $1-\delta/2$,
$$
\max_{1\le i\le n}\|\xi_i^V\|_2
\le
\sigma_V\sqrt{2\log\frac{4n}{\delta}}
$$
Intersecting these two events gives a probability of at least $1-\delta$. Substituting into the previous bound yields
$$
\|y^{\mathrm{reuse}}-y^\star\|_2
\le
\frac{V_{\max}\|q\|_2}{2\sqrt d}
\left(\|\mu_K\|_2+\sigma_K\sqrt{2\log\frac{4n}{\delta}}\right)
+
\left(\|\mu_V\|_2+\sigma_V\sqrt{2\log\frac{4n}{\delta}}\right)
$$

\begin{flushright} $\blacksquare$ \end{flushright}

\subsection{Proposition 2: Error Bound After Probe-Guided Mean Correction}
\label{prop2}
\begin{prop}[Error Bound After Probe-Guided Mean Correction]
Let the setup and assumptions of Proposition~1 hold. Let $S$ denote $b$ probe samples drawn independently with replacement from the chunk residual population, and define the probe mean estimators

$$
\hat\mu_K := \frac{1}{b}\sum_{j\in S} r_j^K,
\qquad
\hat\mu_V := \frac{1}{b}\sum_{j\in S} r_j^V
$$
For all tokens, define the corrected keys and values by
$$
k_i^{\mathrm{corr}} := k_i^{\mathrm{reuse}}+\hat\mu_K
\qquad
v_i^{\mathrm{corr}} := v_i^{\mathrm{reuse}}+\hat\mu_V
$$
Define the corrected residuals
$$
r_i^{K,\mathrm{corr}} := k_i^\star-k_i^{\mathrm{corr}}
\qquad
r_i^{V,\mathrm{corr}} := v_i^\star-v_i^{\mathrm{corr}}
$$
Let
$$
e_K := \hat\mu_K-\mu_K,
\qquad
e_V := \hat\mu_V-\mu_V
$$
Then for all $i$,
$$
r_i^{K,\mathrm{corr}} = \xi_i^K-e_K,
\qquad
r_i^{V,\mathrm{corr}} = \xi_i^V-e_V
$$
Define
$$
z_i^{\mathrm{corr}} := \frac{q^\top k_i^{\mathrm{corr}}}{\sqrt d},
\qquad
\alpha_i^{\mathrm{corr}} := \frac{\exp(z_i^{\mathrm{corr}})}{\sum_{m=1}^n \exp(z_m^{\mathrm{corr}})},
\qquad
y^{\mathrm{corr}} := \sum_{i=1}^n \alpha_i^{\mathrm{corr}} v_i^{\mathrm{corr}}
$$
Then for any $0<\delta<1$, with probability at least $1-\delta$,
$$
\|y^{\mathrm{corr}}-y^\star\|_2
\le
\frac{V_{\max}\|q\|_2}{2\sqrt d}
\left(
\sigma_K\sqrt{2\log\frac{8n}{\delta}}
+
\sigma_K\sqrt{\frac{2\log(8/\delta)}{b}}
\right)
+
\left(
\sigma_V\sqrt{2\log\frac{8n}{\delta}}
+
\sigma_V\sqrt{\frac{2\log(8/\delta)}{b}}
\right)
$$
\end{prop}

\textit{Proof. }
For analytical simplicity, the proposition considers the uniform correction applied to all tokens, including probe tokens. The practical algorithm is at least as favorable, since probe tokens can instead use fresh KV directly.

Since the centered token-wise residuals are mean-zero and sub-Gaussian, and the probe tokens are sampled independently with replacement from the chunk residual population, standard concentration of sample means implies that, for all $t>0$,
$$
\Pr(\|e_K\|_2 \ge t) \le 2e^{-b t^2/(2\sigma_K^2)},
\qquad
\Pr(\|e_V\|_2 \ge t) \le 2e^{-b t^2/(2\sigma_V^2)}
$$

For the corrected keys and values, apply Lemma~1 with
$$
k_i = k_i^{\mathrm{corr}},
\qquad
v_i = v_i^{\mathrm{corr}}
$$
Then
$$
\|y^{\mathrm{corr}}-y^\star\|_2
\le
\frac{V_{\max}\|q\|_2}{2\sqrt d}
\max_{1\le i\le n}\|k_i-k_i^\star\|_2
+
\max_{1\le i\le n}\|v_i-v_i^\star\|_2
$$
Equivalently,
$$
\|y^{\mathrm{corr}}-y^\star\|_2
\le
\frac{V_{\max}\|q\|_2}{2\sqrt d}
\max_{1\le i\le n}\|r_i^{K,\mathrm{corr}}\|_2
+
\max_{1\le i\le n}\|r_i^{V,\mathrm{corr}}\|_2
$$
For all tokens,
$$
r_i^{K,\mathrm{corr}}=\xi_i^K-e_K,
\qquad
r_i^{V,\mathrm{corr}}=\xi_i^V-e_V
$$
Thus, by the triangle inequality,
$$
\max_{1\le i\le n}\|r_i^{K,\mathrm{corr}}\|_2
\le
\max_{1\le i\le n}\|\xi_i^K\|_2+\|e_K\|_2,
$$
$$
\max_{1\le i\le n}\|r_i^{V,\mathrm{corr}}\|_2
\le
\max_{1\le i\le n}\|\xi_i^V\|_2+\|e_V\|_2
$$

By the same union-bound argument as in Proposition~1, with probability at least $1-\delta/4$,
$$
\max_{1\le i\le n}\|\xi_i^K\|_2
\le
\sigma_K\sqrt{2\log\frac{8n}{\delta}},
$$
and with probability at least $1-\delta/4$,
$$
\max_{1\le i\le n}\|\xi_i^V\|_2
\le
\sigma_V\sqrt{2\log\frac{8n}{\delta}}
$$

Next, by the above sample-mean concentration bound, with probability at least $1-\delta/4$,
$$
\|e_K\|_2
\le
\sigma_K\sqrt{\frac{2\log(8/\delta)}{b}}
$$
and with probability at least $1-\delta/4$,
$$
\|e_V\|_2
\le
\sigma_V\sqrt{\frac{2\log(8/\delta)}{b}}
$$

Taking the intersection of these four events gives probability at least $1-\delta$. On this event,
$$
\max_{1\le i\le n}\|r_i^{K,\mathrm{corr}}\|_2
\le
\sigma_K\sqrt{2\log\frac{8n}{\delta}}
+
\sigma_K\sqrt{\frac{2\log(8/\delta)}{b}}
$$
$$
\max_{1\le i\le n}\|r_i^{V,\mathrm{corr}}\|_2
\le
\sigma_V\sqrt{2\log\frac{8n}{\delta}}
+
\sigma_V\sqrt{\frac{2\log(8/\delta)}{b}}
$$
Substituting into the deterministic perturbation bound yields
$$
\|y^{\mathrm{corr}}-y^\star\|_2
\le
\frac{V_{\max}\|q\|_2}{2\sqrt d}
\left(
\sigma_K\sqrt{2\log\frac{8n}{\delta}}
+
\sigma_K\sqrt{\frac{2\log(8/\delta)}{b}}
\right)
+
\left(
\sigma_V\sqrt{2\log\frac{8n}{\delta}}
+
\sigma_V\sqrt{\frac{2\log(8/\delta)}{b}}
\right)
$$

\begin{flushright} $\blacksquare$ \end{flushright}

\subsection{Proposition 3: Sufficient condition for a tighter bound under correction}
\label{prop3}
\begin{prop}[Sufficient condition for a tighter bound under correction]
If proposition 1 and 2 each hold with probability at least $1 - \delta$, then by a union bound, they hold simultaneously with probability at least $1 - 2 \delta$. Define
$$
C_K := \frac{V_{\max}\|q\|_2}{2\sqrt d},
\qquad
C_V := 1
$$
Also define
$$
B := C_K\|\mu_K\|_2 + C_V\|\mu_V\|_2
\qquad \text{(common bias)}
$$
$$
N_n := C_K\sigma_K\sqrt{2\log\frac{4n}{\delta}}
      + C_V\sigma_V\sqrt{2\log\frac{4n}{\delta}}
\qquad \text{(fluctuation term before correction)}
$$
$$
E_b := C_K\sigma_K\sqrt{\frac{2\log(8/\delta)}{b}}
      + C_V\sigma_V\sqrt{\frac{2\log(8/\delta)}{b}}
\qquad \text{(probe estimation term)}
$$
and
$$
N_n' := C_K\sigma_K\sqrt{2\log\frac{8n}{\delta}}
      + C_V\sigma_V\sqrt{2\log\frac{8n}{\delta}}
\qquad \text{(fluctuation term after correction)}
$$

Define the corresponding high-probability upper bounds
$$
U_{\mathrm{reuse}} := B + N_n,
\qquad
U_{\mathrm{corr}} := N_n' + E_b
$$

Then, on the same event of probability at least $1-2\delta$,
$$
\|y^{\mathrm{reuse}}-y^\star\|_2 \le U_{\mathrm{reuse}},
\qquad
\|y^{\mathrm{corr}}-y^\star\|_2 \le U_{\mathrm{corr}}
$$
Consequently,
$$
U_{\mathrm{reuse}} - U_{\mathrm{corr}}
=
B - E_b - (N_n' - N_n)
$$
In particular, if
$$
B > E_b + (N_n' - N_n)
$$
then
$$
U_{\mathrm{corr}} < U_{\mathrm{reuse}}.
$$
That is, the correction method has a strictly tighter high-probability upper bound than uncorrected reuse. Moreover, the gap between the two upper bounds is
$$
\Delta_b := U_{\mathrm{reuse}} - U_{\mathrm{corr}}
= B - E_b - (N_n' - N_n)
$$
\end{prop}

\textit{Proof. }
By Proposition~1,
$$
\|y^{\mathrm{reuse}}-y^\star\|_2
\le
C_K\|\mu_K\|_2 + C_V\|\mu_V\|_2
+
C_K\sigma_K\sqrt{2\log\frac{4n}{\delta}}
+
C_V\sigma_V\sqrt{2\log\frac{4n}{\delta}}
$$
which is exactly
$$
\|y^{\mathrm{reuse}}-y^\star\|_2 \le U_{\mathrm{reuse}} = B + N_n
$$

Similarly, by Proposition~2,

$$
\|y^{\mathrm{corr}}-y^\star\|_2
\le
C_K\sigma_K\sqrt{2\log\frac{8n}{\delta}}
+
C_V\sigma_V\sqrt{2\log\frac{8n}{\delta}}
+
C_K\sigma_K\sqrt{\frac{2\log(8/\delta)}{b}}
+
C_V\sigma_V\sqrt{\frac{2\log(8/\delta)}{b}}
$$

which is exactly
$$
\|y^{\mathrm{corr}}-y^\star\|_2 \le U_{\mathrm{corr}} = N_n' + E_b
$$

Subtracting the two \emph{upper bounds} gives
$$
U_{\mathrm{reuse}} - U_{\mathrm{corr}}
=
(B+N_n) - (N_n' + E_b)
=
B - E_b - (N_n' - N_n)
$$
Therefore, if
$$
B > E_b + (N_n' - N_n)
$$
then
$$
U_{\mathrm{corr}} < U_{\mathrm{reuse}}
$$
Hence the corrected method has a strictly tighter high-probability upper bound than uncorrected reuse. The stated expression for $\Delta_b$ is just the difference between these two upper bounds.

\begin{flushright} $\blacksquare$ \end{flushright}

Intuitively, with high probability, the uncorrected reuse error is bounded by
$$
\text{common bias} + \text{token-wise fluctuation},
$$
while the probe-corrected error is bounded by
$$
\text{probe estimation error} + \text{token-wise fluctuation}.
$$
Therefore, probe-guided correction yields a tighter high-probability error bound whenever the shared chunk-level residual bias dominates the error of estimating that bias from the probe set.

\section{Empirical Validation of the Modeling Assumptions}
\label{emp_val}

\subsection{Mean-explained Energy and Change in Top Singular Value of Residual}
\label{mean_explained}

\begin{figure}
    \centering
    \includegraphics[width=1\linewidth]{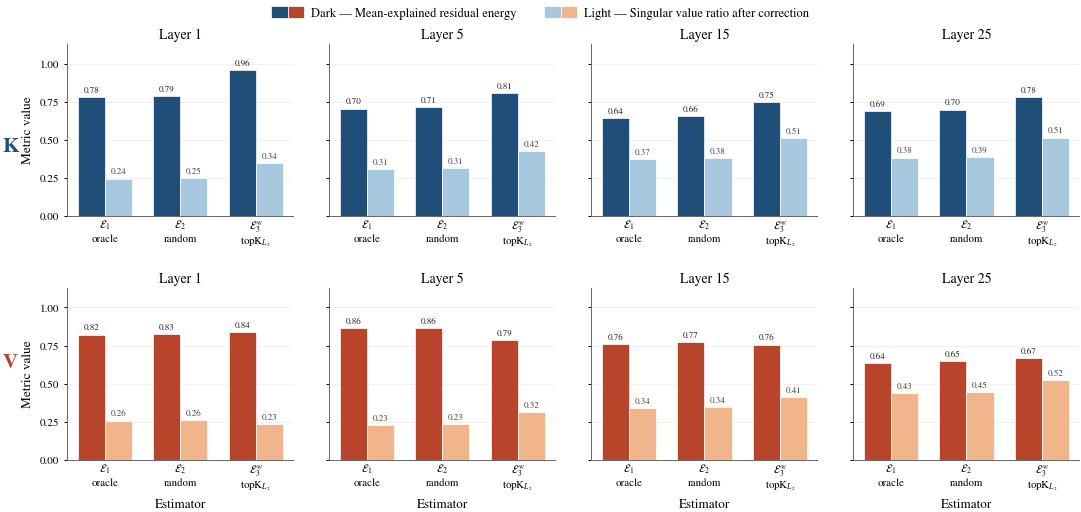}
    \caption{Mean-explained residual energy (dark) and remaining top singular value ratio (light) after correction across layers for Key(top) and Value(bottom), comparing oracle($\mathcal{E}_1$), random($\mathcal{E}_2$), and topK$_{L_2}$($\mathcal{E}_3^w$) correction vectors. We use $b=10\% \text{of tokens}$ for $\mathcal{E}_2$ and $\mathcal{E}_3^w$.} 
    \label{fig:correction}
\end{figure}

Figure~\ref{fig:correction} empirically supports the residual decomposition $r_i=\mu+\xi_i$ underlying our analysis. We compare three correction vectors: $\mathcal{E}_1$ (oracle), which uses the exact chunk residual mean; $\mathcal{E}_2$ (random), which estimates the mean from randomly sampled probe tokens as in our theory; and $\mathcal{E}_3^w$ (topK), our practical heuristic based on the most divergent tokens. Here, mean-explained residual energy measures how much of the total residual energy is removed by the correction vector, while the top singular value ratio after correction measures how much of the dominant structured residual mode remains, with smaller values indicating stronger suppression. Across layers, the oracle and random estimators behave very similarly: for both K and V, they explain a large fraction of residual energy ($86\%\!-\!64\%$) and reduce the top singular value after correction to only about $24\%\!-\!45\%$ of its original magnitude. This indicates that the dominant residual structure is largely a shared chunk-level offset, and that the probe-based estimator used in our theory is already sufficient to recover it in practice. The topK heuristic often explains slightly more residual energy ($96\%\!-\!67\%$), which is better from a KV reuse perspective, while leaving a somewhat larger remaining leading singular mode after correction, suggesting that it captures not only the shared mean offset but also additional structured variation that is beneficial in practice.

\subsection{Empirical Validation of the Sub-Gaussian Fluctuation Assumption}
\label{fluctuation_validation}

The analysis in Section~\ref{sec:theory} models each residual block as a sum of a shared chunk-level offset and a centered token-wise fluctuation term, $r_i=\mu+\xi_i$. While Appendix section~\ref{mean_explained} empirically validates the shared-mean component, propositions~1 to 3 additionally rely on the assumption that the centered fluctuation term is sub-gaussian (sufficiently light-tailed) so that probe-mean estimation error concentrates as the probe budget increases. This appendix subsection evaluates whether the residual blocks observed in practice behave consistently with that concentration-based picture.

Our goal is to test whether the token-wise fluctuations behave in a way consistent with the theory. For our correction analysis, the most important implication is that the probe-mean estimator should become more accurate as the number of probes $b$ increases, ideally at the standard sample-mean rate. A second implication is that, after removing the chunk mean, the remaining token-wise fluctuations should look reasonably light-tailed in one-dimensional projections. To examine these properties, we use the saved per-block residual tensors, where each row corresponds to the residual of one token. We analyze K and V separately. For each block, we first compute the block mean $\bar r$ and define the centered fluctuation block $R-\mathbf{1}\bar r^\top$. We then perform two diagnostics.

First, we study probe-mean error versus probe budget. For each block and each probe budget $b$, we repeatedly sample $b$ tokens with replacement, compute the sample mean $\hat\mu_b$, and measure the estimation error $\|\hat\mu_b-\bar r\|_2$. Aggregating these errors across blocks gives an empirical view of how quickly the probe-estimated correction vector approaches the full block mean as $b$ increases. This directly tests the concentration behavior used in the propositions. Second, we study the distribution of the centered fluctuations through random projection to 1D. For each residual block, we draw random unit directions $u$ and compute the projected fluctuations $u^\top(r_i-\bar r)$. Pooling these projections across blocks yields a one-dimensional view of the fluctuation distribution. We summarize this distribution with QQ plots against a standard normal reference and with tail plots of $\log \Pr(|Z|>t)$ versus $t^2$. These diagnostics do not establish exact sub-Gaussianity, but they reveal whether the centered fluctuations are at least consistent with a light-tailed, approximately sub-Gaussian concentration model. For both experiments and both K and V, the results are shown in Figure~\ref{fig:sub_gaussian_val}.

\begin{figure}
    \centering
    \includegraphics[width=1\linewidth]{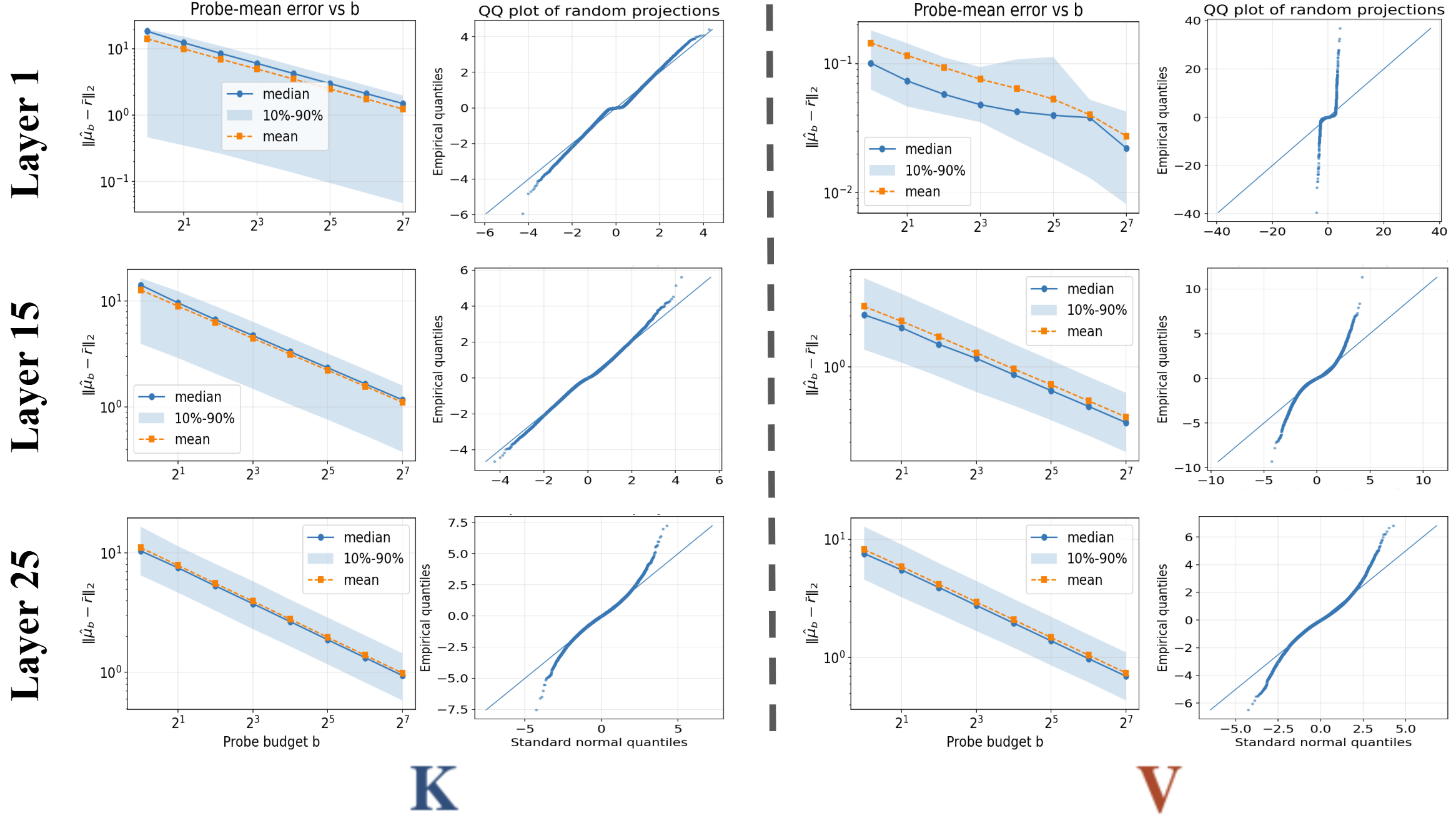}
    \caption{Empirical diagnostics of the token-wise fluctuation term for K (left) and V (right) across Layers~1, 15, and 25. For each layer, the left panel shows probe-mean estimation error versus probe budget $b$. The right panel shows QQ plots of random projections of the centered residuals against a Gaussian reference, where closer alignment with the diagonal indicates behavior more consistent with a light-tailed concentration model.} 
    \label{fig:sub_gaussian_val}
\end{figure}

For K residuals, the empirical diagnostics are strongly consistent with the concentration-based picture underlying propositions~1 to 3. Across Layers~1, 15, and 25, the probe-mean estimation error $\|\hat{\mu}_b-\bar r\|_2$ decreases smoothly as the probe budget $b$ increases and is close to linear on log--log axes, indicating behavior near the standard $b^{-1/2}$ sample-mean rate. In parallel, the QQ plots of centered random projections remain close to the Gaussian reference line, with only moderate deviations in the extreme tails. Taken together, these results suggest that the token-wise fluctuation for K approximates well by a light-tailed concentration model, and they provide strong empirical support for the probe-estimation and token-wise fluctuation assumptions used in our Key side analysis.

For V residuals, the empirical evidence is more mixed than for K, but still broadly supportive of the theory. Across Layers~1, 15, and 25, the probe-mean estimation error decreases monotonically with probe budget $b$, indicating that the probe-estimated correction vector becomes increasingly accurate as more probes are used. Moreover, the QQ plots show a clear depth-wise trend: although Layer~1 exhibits non-Gaussian behavior, the distributions in Layers~15 and 25 move progressively closer to the Gaussian reference line. Thus, while the sub-Gaussian-style assumption is less accurate for V than for K, it becomes increasingly reasonable in deeper layers, where the empirical behavior is more consistent with the concentration-based picture used in the analysis.

Overall, the empirical results here largely validate the modeling assumptions and correction mechanism underlying our analysis, especially for K. At the same time, we note that the residual characteristics of K and V can be different under KV reuse: K is more cleanly aligned with the concentration-based model, while V exhibits a more structured and layer-dependent behavior. This suggests that K and V may benefit from different reuse or correction mechanisms, and points to a promising direction for future work on asymmetric KV correction strategies tailored separately to keys and values.

\section{Societal Impact}
AgentKVShift reduces the prefill cost of memory-augmented LLM agents, which has both positive and neutral societal implications. On the positive side, lower inference latency and compute reduce the energy footprint of long-horizon
LLM serving and broaden access to memory-augmented agents in resource-constrained environments. On the other hand, like other efficient LLM serving research, our work is subject to the dual-use risk that reducing inference cost may lower the barrier to deploying LLM-based agents at scale, including in
applications that raise privacy, fairness, or misuse concerns. We do not anticipate any direct negative societal impact unique to this work beyond those already present in the systems it accelerates.

\section{Limitations and Future Work}
\label{sec:limitations}

While AgentKVShift offers an effective method for KV reuse in agentic memory systems, our work has point to promising directions for future research.

\paragraph{Asymmetry between K and V residuals.}
As discussed in our empirical analysis (Appendix~\ref{fluctuation_validation}), we show that the sub-Gaussian fluctuation assumption used in our theoretical bounds fits keys more cleanly than values, especially in shallow transformer layers. This suggests that asymmetric correction strategies, applying different reuse
mechanisms to K and V or using layer-dependent probe budgets, could further improve quality. We leave a careful study of K-V asymmetric KV correction to future work.

\paragraph{Cross-chunk memory reasoning.}
Our extended results on AMA-Bench (Appendix~\ref{app:full_amabench}) highlight that
all KV reuse methods, including AgentKVShift, leave a wider gap to full recompute on cross-chunk reasoning capabilities such as Causal Inference and State Updating. Since these capabilities depend more on the agent's reasoning
pipeline than on chunk-local fidelity, they are not fully addressed by per-memory correction. A promising direction is state-aware KV reuse, where the correction is conditioned on the agent's current reasoning state rather than treated as purely memory-local.



\end{document}